  \documentclass[conference, nofonttune, english, letter]{ieeetran}
\IEEEoverridecommandlockouts                              
\usepackage[dvipsnames]{xcolor}
\usepackage{graphicx}
\usepackage{cite}
\usepackage{graphics} 
\usepackage{epsfig} 
\usepackage{times} 
\usepackage{amsmath} 
\usepackage{amssymb}  
\usepackage{subcaption}
\usepackage{hyperref}
\usepackage[ruled]{algorithm2e}
\usepackage{siunitx} 
\usepackage{booktabs} 
\usepackage{makecell}
\usepackage{orcidlink}
\usepackage{multirow}
\usepackage{color, colortbl}
\usepackage{adjustbox}
\usepackage{esvect}
\usepackage{tikz}
\usepackage{algpseudocode}
\usepackage{url}
\usepackage{amsthm}

\definecolor{Gray}{gray}{0.9}

\newcommand{\neural}{\textit{Neural-L1}}
\newcommand{\ladap}{\textit{L1adap}}
\newcommand{\deepM}{\textit{Deep-MRAC}}
\newcommand{\LF}{\textit{LF}}

\graphicspath{{fig/}}
\title{\LARGE \bf Neural $\mathcal{L}_1$ Adaptive Control of Vehicle Lateral Dynamics}
\author{Pratik Mukherjee$^{1}$\orcidlink{0000-0003-2970-8515}, Burak M. Gonultas$^{1}$\orcidlink{0000-0002-7966-7929}, O. Goktug Poyrazoglu$^{1}$\orcidlink{0009-0002-3778-100X}  and Volkan Isler$^{1}$\orcidlink{0000-0002-0868-5441}
\thanks{$^{1}$ The authors are with the Department of Computer Science and Engineering, University of Minnesota, Minneapolis, MN, 55455, USA
        {\tt\small \{mukhe027, gonul004, poyra002, isler\}@umn.edu}}%
}

\newtheorem{theorem}{Theorem}

\newtheorem{lemma}{Lemma}

\newtheorem{assumption}{Assumption}

\newtheorem{problem}{Problem}

\newtheorem{fact}{Fact}

\begin{document}


\maketitle
\thispagestyle{empty}
\pagestyle{empty}

\begin{abstract}
We address the problem of stable and robust control of vehicles with lateral error dynamics for the application of \emph{lane keeping}. Lane departure is the primary reason for half of the fatalities in road accidents, making the development of \emph{stable}, \emph{adaptive} and \emph{robust} controllers a necessity. 
 Traditional \emph{linear feedback} controllers achieve satisfactory tracking performance, however, they exhibit unstable behavior when \emph{uncertainties} are induced into the system. Any disturbance or uncertainty introduced to the steering-angle input can be catastrophic for the vehicle. Therefore, controllers must be developed to actively handle such uncertainties. 
In this work, we introduce a Neural $\mathcal{L}_1$ Adaptive controller (\neural) which learns the uncertainties in the lateral error dynamics of a front-steered Ackermann vehicle and guarantees stability and robustness. Our contributions are threefold: i)~We extend the theoretical results for \emph{guaranteed} stability and robustness of conventional $\mathcal{L}_1$ Adaptive controllers to \neural; ii)~We implement a \neural~for the lane keeping application which learns uncertainties in the dynamics accurately; iii)~We evaluate the performance of \neural~on a physics-based simulator, PyBullet, and conduct extensive real-world experiments with the F1TENTH platform 
to demonstrate superior reference trajectory \emph{tracking performance} of \neural~compared to other state-of-the-art controllers, in the presence of uncertainties. Our project page, including supplementary material and videos, can be found at \url{ https://mukhe027.github.io/Neural-Adaptive-Control/}
\end{abstract}

\section{Introduction}

\IEEEPARstart{S}{afety critical} systems are prevalent in the transportation industry, especially in the aerospace sector. Therefore, the focus has always been on automating flight control systems because of the underlying complexities of flying an aircraft \cite{hovakimyan20111,annaswamy2013recent,wise2006adaptive}. On the contrary, in the automotive sector, road vehicles have traditionally been driven manually, primarily dependent on a human making decisions behind the wheel. For automotive systems, the lower level vehicle control system relies on conventional controllers such as the \emph{linear state-feedback} controller (\LF)~\cite{rajamani2011vehicle}. However, with recent surge in interest for developing \emph{autonomous} road vehicles, there is a need for more reliable control systems that are robust to induced uncertainties.


Federal Highway Administration has stated that from 2016 to 2018 an average of 19,158 fatalities resulted from roadway departures, which is 51 percent of all traffic fatalities in the United States \cite{FHWA}. 
 Lane departures are identified to be the number one cause of fatal accidents in the United States by the National Highway Transportation Safety Administration. 
Thus, there has been significant investment in research in industry as well as in academia for developing systems that recognize and identify accidental lane departures of vehicles \cite{rajamani2011vehicle,bevly2016lane}. In this direction, some of the research thrusts are: \emph{lane departure warning systems} (LDWS), \emph{lane keeping systems} (LKS) and yaw stability control systems. We particularly focus on LKS in this paper and propose to develop not just a \emph{stable} controller but also an \emph{adaptive} controller which is \emph{robust} to uncertainties induced into a system in the form of sensor noise, control signal disturbances and external obstacles on road. 
\begin{figure}[!htb]
\centering
\subfloat{\includegraphics[width=5cm,height=8cm,, angle =90,keepaspectratio]{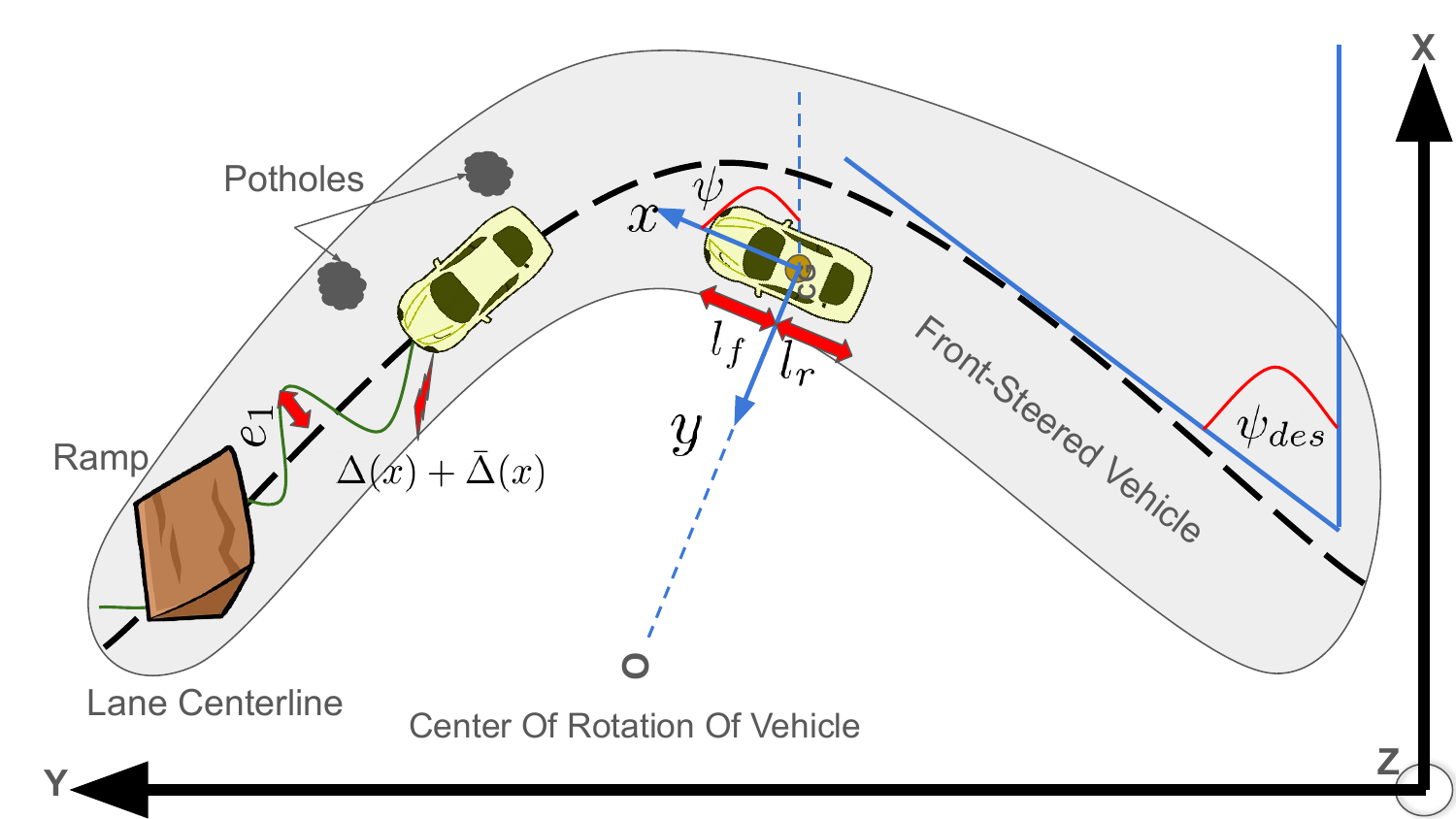}}
\caption{A lateral lane keeping system (LKS) for a front-steered Ackermann vehicle, depicting the effect of uncertainties such as signal disturbance in the form $\Delta(x)+\Bar{\Delta}(x)$, potholes and ramp, on the lane keeping performance. The image depicts a simplified linear two-degrees of freedom (2-DOF) bicycle model of the vehicle lateral dynamics
derived in \cite{rajamani2011vehicle}.}
\vspace{-10pt}
\label{fig:lateral_vehicle_dynamics_framework}
\end{figure}

With the need for certifying advanced adaptive flight critical systems emerging rapidly, the seminal work by Hovakimyan et al. \cite{hovakimyan20111} introduces $\mathcal{L}_1$ adaptive controller (\ladap). \ladap~\cite{hindman2007designing,cao2008adaptive,cao2008design,wang2008l1,wang2009l1, leman2009l1,cao20091} is a class of control system which is inherently derived from the traditional class of \emph{adaptive} controllers, \emph{Model Reference Adaptive} Controllers (MRAC) \cite{ioannou1996robust,narendra2012stable,annaswamy1995discrete} and the \emph{ State Predictor-based MRAC} \cite{slotine1991applied}. Beyond the traditional properties of guaranteed stability and \emph{adaptiveness} of MRAC type controllers, \ladap~provides \emph{robustness} guarantees to uncertainties induced into the system dynamics. In this paper, we show that the performance of traditional \ladap \cite{hovakimyan20111} can be further enhanced with the help of neural network. \deepM~in \cite{joshi2019deep} has demonstrated that the controller performance of traditional MRAC is further improved by using neural network to derive the \emph{adaptive laws} that learn the uncertainties better, instead of using conventional \emph{adaptive laws} of MRAC to estimate the \emph{true} uncertainties.
Therefore, we present the implementation of a modified \ladap, \emph{Neural} $\mathcal{L}_1$ Adaptive controller (\neural), where the adaption law is learned by learning the weights of neural network using the framework in \cite{joshi2019deep}, and then appended to the adaptive laws derived from the conventional theory of \ladap~in \cite{hovakimyan20111}. Further, we extend the proof for stability and robustness of the \ladap~\cite{hovakimyan20111} with this new neural network framework to \neural. Finally, we evaluate the performance of \neural~for a lane-keeping application using a physics-based simulation of the F1TENTH front-steered Ackermann vehicle as well as on a real F1TENTH vehicle. Our evaluations show that the proposed \neural~is robust and stable in the presence of uncertainties, whereas other existing controllers such as the \LF~\cite{rajamani2011vehicle} and the \deepM~controller from \cite{joshi2019deep} exhibit unstable behavior and have a deteriorating tracking performance. 
When the tracking performance of our proposed controller is compared with that of the conventional \ladap, the use of neural network to learn the \emph{residual} uncertainty enhances the tracking performance of our proposed controller. 
To summarise, our contributions are threefold: i) We extend the theoretical results for \emph{guaranteed} stability and robustness of conventional \ladap~to \neural; ii) We implement a \neural~for lane keeping application of front-steered Ackermann vehicles which uses neural network to learn induced uncertainties in the system dynamics; iii) We evaluate the performance of \neural~on a physics-based simulator, PyBullet, and conduct extensive real-world experiments with the F1TENTH platform 
to demonstrate superior reference trajectory \emph{tracking performance} of \neural~compared to other state-of-the-art controllers. For a rigorous evaluation of all controllers, we externally induce sensor noise, control signal disturbance and place physical obstacles in environment. \neural~exhibits stability and robustness to induced noise and disturbance to achieve successful tracking of \emph{arbitrary} reference trajectories.

The rest of the paper is organized as follows. Related work is summarized in
Section \ref{Related Work}. The formal problem formulation is given in Section \ref{Problem Formulation}. In Section \ref{System Overview}, we provide extensive details of the \emph{stability} theory behind \neural. Finally, we provide simulation and experimental results with F1TENTH in Sections \ref{SIM-Results} and \ref{EXP-Results} and conclude the paper in Section \ref{end}.

\section{Related Work}\label{Related Work}

The works \cite{peng2000vehicle,gohl2002development,kidane2005control,lewis2003sliding,rajamani2002semi} are some of the initial efforts that discuss control methods for maintaining stability in road vehicles for prevention of rollover and lateral control of front-steered vehicles. With respect to applications such as LKS and LDWS, lateral control of vehicle is of importance.  \cite{rajamani2003lateral,taylor1999comparative,wang2001trajectory,thorpe1987vision,mouri1997automatic,akar2008lateral,naranjo2008lane,cai2003robust } are works which discuss the application of stable lateral control in road vehicles and \cite{ keviczky2006predictive, falcone2007predictive,lee2009improved} are more recent works that address the problem of stable steering control in vehicles. In \cite{taylor1999comparative}, the authors provide one of the first studies of application of \emph{vision}-based systems in lateral control of vehicles. In work \cite{wang2001trajectory}, authors develop a two-part trajectory planning algorithm which consists of the steering planning and velocity planning for a four-wheeled steering vehicle. 
However, none of these works address the issue of controller \emph{stability}, \emph{adaptiveness} and \emph{robustness} in the presence of external factors such as sensor noise, control signal disturbance or physical obstacles on road.

 With the interest growing in developing LKS and LDWS technologies, there is a need for controllers that can \emph{adapt} to uncertainties induced externally into the system. 
 Most Proportional-Integral-Derivative (PID) controllers should be sufficient to implement LKS and LDWS systems as demonstrated in these works \cite{ marino2009nested, zang2007fuzzy,haytham2014modeling}. In this regard, \emph{linear state-feedback} controllers such as PID controllers are usually the primary choice due to their fairly robust nature and the ease of implementation on real systems, but since control systems are prone to uncertainties being induced from the environment, PID controllers are not capable of adapting in real-time to these uncertainties efficiently. 
  Therefore, this naturally drives the interest for developing \emph{adaptive} controllers. 
Adaptive controllers have been around for a while \cite{ hammer1962adaptive, zinober1980adaptive,morse1979global,rohrs1982adaptive}, but they are primarily implemented for highly safety-critical aerospace systems \cite{morse1989flight,joshi2019deep}. In \cite{morse1989flight}, the authors, for the first time, implement an MRAC for the flight control system of F-16 aircraft. The work in \cite{joshi2019deep} is relatively new which shows stable control of aerial vehicles, where the \emph{adaptive laws} of the MRAC are learned using a neural network. The MRAC framework is adaptive and stable, but it is not robust with respect to uncertainties in the dynamics of the system. In this paper, we adopt the neural network architecture from \cite{joshi2019deep} and apply it to a traditional \ladap~to account for \emph{robustness}. The key feature of \ladap~\cite{hovakimyan20111} architecture is that it guarantees robustness in the presence of fast adaptation, which leads to uniform performance bounds both in transient and steady-state operation. 
In the recent works \cite{wu20221, huang2023datt}, the authors demonstrate the augmentation of \ladap~to other conventional controllers. The work \cite{wu20221} augments  \ladap~to a \emph{geometric} controller for control of quadrotors. The authors demonstrate that the robustness characteristic of \ladap~handles non-linear uncertainties induced into the system, whereas \cite{huang2023datt} augments the same with a feedforward-feedback control system and learns a stable controller for adaptive trajectory tracking of \emph{unfeasible} trajectories by quadrotors. None of these works address the LKS application.
In \cite{shirazi2017}, the authors provide \emph{only} simulation results corroborating the characteristics of stability, robustness and adaptability claimed by the seminal work \cite{hovakimyan20111} for the application of lane keeping using the lateral error dynamics in \cite{rajamani2011vehicle}. 


\section{Problem Formulation}\label{Problem Formulation}
In this section, we review the theory behind the front-steered Ackermann vehicle lateral \textit{error} dynamics from \cite{rajamani2011vehicle} and the \emph{neural network} approach we adopt from \cite{joshi2019deep}.



 \subsection{Vehicle Lateral Error Dynamics}
 Fig.~\ref{fig:lateral_vehicle_dynamics_framework} is a depiction of a simplified linear two-degrees of freedom (2-DOF) bicycle  model of the vehicle lateral dynamics derived in \cite{rajamani2011vehicle}.
 2-DOF is represented by $\psi$, the vehicle yaw angle, and $y$, the lateral position with respect to the center of the rotation of the vehicle $O$.  The yaw angle is considered as the angle between horizontal \textit{vehicle body frame} axis of the vehicle, $x-axis$, and the global horizontal axis, $(X)$. The \textit{constant} longitudinal velocity of the vehicle at its center of gravity ($\boldsymbol{CG}$) is denoted by $V_x$ and the mass of the vehicle is denoted by $m$. The distances of the front and rear axles from the $\boldsymbol{CG}$ are shown by $l_f$, $l_r$, respectively, and the front and rear tire cornering stiffness are denoted by $C_{af}$ and $C_{ar}$, respectively. The steering angle is denoted by $\delta$ which also serves as the control signal when a controller is implemented and the yaw moment of inertia of the vehicle is denoted by $I_z$. Considering the lateral position, yaw angle, and their derivatives as the state variables, and using the Newton's second law for the motion along the \textit{vehicle body frame} $y-axis$, the state space model of lateral vehicle dynamics is derived in \cite{rajamani2011vehicle}. 
 
The error dynamics is written with two error variables : $e_1$, which is the distance between the $\boldsymbol{CG}$ of the vehicle from the center line of the lane; $e_2$, which is the orientation error of the vehicle with respect to the desired yaw angle $\psi_{des}$.
Assuming the radius of the road is $R$, the rate of change of the desired orientation of the vehicle can be defined as $\dot{\psi}_{des}= \frac{V_x}{R}$.
The tracking or lane keeping objective of the lateral control problem is expressed as a problem  of stabilizing the following \textit{error} dynamics at the origin.
 \begin{align} \label{eq:er_dyn_st_sp}
&  \underbrace{\frac{d}{dt}\begin{bmatrix}
e_1  \\
\dot{e_1} \\
e_2\\
\dot{e_2} 
\end{bmatrix}}_{\dot{x}} =
\resizebox{0.75\hsize}{!}{  $\underbrace{\begin{bmatrix}
0 & 1 & 0 & 0\\
0 & -\frac{2 C_{af}+2 C_{ar}}{m V_x} & \frac{2 C_{af}+2 C_{ar}}{m } &  - \frac{2 C_{af}l_f-2 C_{ar}}{m V_x} l_r\\
0 & 0 & 0 & 1\\
0 & -\frac{2 C_{af}l_f-2 C_{ar}l_r}{I_z V_x} & \frac{2 C_{af}l_f-2 C_{ar}l_r}{I_z } & -\frac{2 C_{af}l_f^2+2 C_{ar}l_r^2}{I_z V_x} 
\end{bmatrix}}_{A}$} && \nonumber\\
&  \times 
\underbrace{\begin{bmatrix}
e_1  \\
\dot{e_1} \\
e_2\\
\dot{e_2} 
\end{bmatrix}}_{x} + \underbrace{\begin{bmatrix}
0  \\
\frac{2 C_{af}}{m} \\
0\\
\frac{2C_{af}l_f}{I_z}
\end{bmatrix}}_{B_1}\underbrace{\delta}_{u} + \underbrace{\begin{bmatrix}
0  \\
-V_x - \frac{2 C_{af}l_f-2 C_{ar}}{m V_x} \\
0\\
-\frac{2 C_{af}l_f^2+2 C_{ar}l_r^2}{I_z V_x}
\end{bmatrix}}_{B_2} \dot{\psi}_{des}
\end{align}
 
Therefore, the above state-space form in \eqref{eq:er_dyn_st_sp} can be represented in the \textit{general} state space form as in \eqref{eq:er_gnrl_st_sp}
\begin{align} \label{eq:er_gnrl_st_sp}
\dot{x}(t) = Ax(t) + B_1(u(t)) + B_2 \dot{\Psi}_{des}
\end{align}
where $x(t)\in \mathbb{R}^n$, $t \geq 0$ is a state vector, $u(t) \in \mathbb{R}^m$, $t \geq 0$ is the control input which in this case is the steering angle $\delta$. The elements of $A \in  \mathbb{R}^{n\times n}$, $B_1, B_2 \in  \mathbb{R}^{n\times m}$ matrices are generally considered known, but in our case we obtain the elements of matrices from our system identification work \cite{gonultas2023system} and we assume pair $(A,B_1)$ is controllable. 

\subsection{Deep Neural Networks (DNN) Architecture}
In this paper, we incorporate the DNN architecture from \cite{joshi2019deep} with the purpose of extracting relevant features of the system uncertainties, which otherwise is not trivial to obtain without the limits on the domain of operation. The idea in machine learning is that a given function can be encoded with weighted combinations of \emph{feature} vector $\Phi \in \mathcal{F}$, s.t $ \Phi(x) = [\phi_1(x), \phi_2(x),\dots, \phi_k(x) ]^T \in \mathbb{R}^k$, and $W^* \in \mathbb{R}^{k \times m}$ a vector of ideal weights s.t $\|y(x) - W^{*^T}\Phi(x)\|_{\infty} < \epsilon(x)$. DNNs utilize composite functions of features arranged in a directed acyclic graph, i.e. $\Phi(x)= \phi_n(\theta_{n-1},\phi_{n-1}(\theta_{n-2}, \phi_{n-2}(\dots)))$ where $\theta_i$'s are the layer weights. In this regard, we implement a dual time-scale learning scheme, same as in \cite{joshi2019deep}, to learn the uncertainties in real time. In the dual time-scale scheme, the weights of the outermost layers are adapted in real time using a generative network, whereas the weights of the inner layers are adapted using batch updates. Unlike the conventional neural network training where the true target values $y \in \mathcal{Y}$ are available for every input $x \in \mathcal{X}$ , in this paper, true system uncertainties as the labeled targets are not available for the network training. Therefore, we use part of the network itself (the last layer) with pointwise weight updated according to the adaptive laws of \ladap~as the generative model for the data. 
 The generative model uncertainty estimates $y = W^T \Phi(x, \theta_1, \theta_2, \dots \theta_{n-1}) = \Delta^{'}(x)$ along with inputs $x_i$ make the training data set $Z^{pmax} = \{x_1, \Delta^{'}(x_i)\}^{pmax}_{i=1}$. The adopted DNN architecture is shown in \emph{green} in Fig.~\ref{fig:adaptive_framework}. As shown in Fig.~\ref{fig:adaptive_framework}, the adopted DNN architecture is appended to the \emph{adaptive laws} learned by the traditional \ladap~in \emph{red}. The reader is referred to \cite{joshi2019deep} for extensive detail on the network architecture and the results evaluating the DNN training, testing and validation performances.

\subsection{Problem Statement}\label{subsec:problem-statement}

In this subsection, we state our problem. We remind the reader that in general conventional stable controllers exist for the LKS application, but with increase in interest for autonomy, development of not only stable, but robust and adaptive controllers is a necessity.  Thus, in this work, we place emphasis on taking an existing stable, robust and adaptive controller, the \ladap, extend this controller's theory while incorporating learning mechanisms of neural network to improve the controller's performance to guarantee to maintain the controller's stability, adaptiveness and robustness properties and apply the modified \ladap~controller to LKS applications.

\begin{problem}\label{problem:1}
Given the lateral error dynamics in~\eqref{eq:er_dyn_st_sp} for front-steered Ackermann vehicles, determine a stable, adaptive and robust control law~$u(t)$ in the form of steering angle control signal $\delta$ such that the front-steered vehicle maintains the lane on a road with constant radius $R$ while adapting to unknown, unmodeled external disturbances such as sensor noise, control signal disturbance and physical obstacles on road, in real time. The adaptive properties of the controller can be further enhanced by incorporating learning mechanisms of neural network.
\end{problem}

The above problem is especially difficult when the nature of the disturbance induced into the system is unknown and hard to model. Existing adaptive controllers derive an adaptive law to learn the disturbance, where using neural network further aids with the learning capability of the controller.

\section{System Overview}\label{System Overview}
In this section, we present the \emph{stability} theory for the \neural~controller. The stability analysis of \neural~is an extension of the stability analysis of \ladap~\cite{hovakimyan20111}. After augmenting the traditional \ladap~controller with the DNN framework from \cite{joshi2019deep} to derive the \neural~controller, the extension of the stability analysis of \ladap~becomes non-trivial. Therefore, in this section, we outline the steps as well as provide theoretical results for a complete stability analysis of our proposed \neural~controller. Fig.~\ref{fig:adaptive_framework} shows the complete \neural~architecture. 
\begin{figure*}[t!]
\centering
    \includegraphics[width=0.9\textwidth]{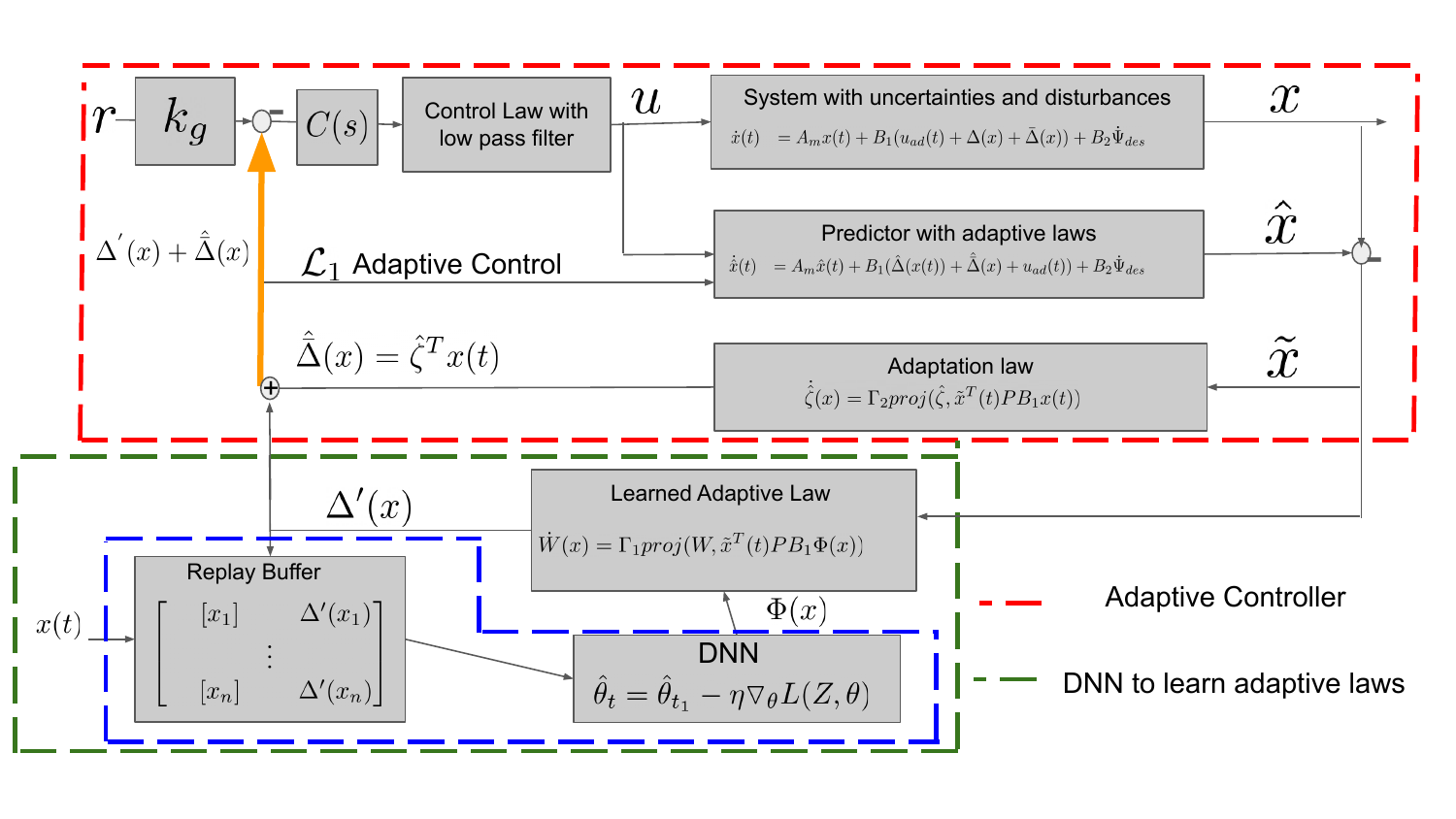}
    \caption{Overview of our approach depicts the \ladap~\cite{hovakimyan20111}  (the red box), the neural network-based adaptive law (the green box) and the neural network training mechanism (the blue box), adopted from \cite{joshi2019deep}. Unlike \cite{joshi2019deep}, in our architecture, we append the neural network learned adaptive law $\Delta^{'}(x)$ to the \ladap~derived adaptive law $\hat{\Bar{\Delta}}(x)$, to obtain $\Delta^{'}(x_1)+\hat{\Bar{\Delta}}(x_1)$ (indicted with yellow arrow). Here neural network learned adaptive law $\Delta^{'}(x)$ is considered to be the residual uncertainty that is being learned.}
    \label{fig:adaptive_framework}
\end{figure*}
We would like to emphasize that our controller architecture of \neural~is different from the controller architecture of \deepM~in \cite{joshi2019deep}. This is because we append the adaptive laws learned by the DNN framework to the \emph{adaptive laws} derived by the \ladap~\cite{hovakimyan20111} controller, unlike in \deepM~\cite{joshi2019deep} where the adaptive laws learned by the DNN framework \emph{replace} the adaptive laws derived from a traditional MRAC \cite{annaswamy1995discrete} controller. The addition of the DNN based adaptive laws to the adaptive laws derived by \ladap~is indicated with an yellow arrow in Fig~\ref{fig:adaptive_framework}. 

The stability analysis of \neural~consists of the following steps. First, we will define an ideal closed-loop reference system $\dot{x}_{ref}(t)$  with control law $u_{ref}(t)$. We will then use the $\mathcal{L}_1$ norm condition to show that $x_{ref}(t)$ is uniformly bounded in Lemma~\ref{lem:l1_cond}. Next, we will analyse the transient and steady state performance of our controller by first defining the error dynamics as $\dot{\Tilde{x}}(t)$, then presenting results of uniform boundedness of $\Tilde{x}(t)$ and convergence $\underset{t\rightarrow \infty}{\lim} \Tilde{x} (t)=LB$ to \emph{lower bound} in Lemma~\ref{lem:pred_er} and Lemma~\ref{lem:conv}, respectively. The error dynamics $\dot{\Tilde{x}}$ is derived as the plant and state predictor error dynamics. We prove $\underset{t\rightarrow \infty}{\lim} \Tilde{x}(t) =LB$ in Lemma~\ref{lem:conv} by proving that the state of the predictor remains bounded. Finally, using the results from Lemma~\ref{lem:pred_er} and Lemma~\ref{lem:conv},  we show the \emph{main} results in Theorem~\ref{thrm:1} proving uniform boundedness of $\|x_{ref}(t)-x(t)\|_{\mathcal{L}_{\infty}}$ and $\|u_{ref}(t)-u(t)\|_{\mathcal{L}_{\infty}}$. The results in Theorem~\ref{thrm:1} imply that by selecting the right parameters for the static feedback gain $k_m$ , the Bounded Input, Bounded Output (BIBO)-stable and strictly proper transfer function of the filter $C(s)$ and the adaptation gain $\Gamma$, we can ensure a desirable response out of the \emph{reference} system and expect the plant states $x(t)$ to approach the \emph{reference} system closely. We conclude this section discussing the proof for Lyapunov stability of the complete controller after the augmentation of the DNN framework.    

\subsection{$\mathcal{L}_1$ Adaptive Control (\ladap) Architecture}
Consider the \emph{total} adaptive control law $u(t)$, consisting of linear feedback component, $u_m(t)$, and the adaptive control component $u_{ad}(t)$
\begin{align}\label{eq:totl_contrl_law}
    u(t) = u_m (t) + u_{ad}(t), ~ u_m(t) = - k_m x(t)
\end{align}
where $k_m \in \mathbb{R}^n$ renders $A_m  \triangleq A-B_1k_m$ Hurwitz, while $u_{ad}(t)$ is the adaptive component, to be defined shortly. The static feedback gain $k_m$, with the control law in \eqref{eq:totl_contrl_law} substituted in \eqref{eq:er_gnrl_st_sp}, leads to the following closed-loop system:
\begin{align} \label{eq:er_unct}
\dot{x}(t) &= A_mx(t) + B_1(u_{ad}(t)+ \Delta(x) + \Bar{\Delta}(x)) + B_2\dot{\Psi}_{des},\nonumber\\ 
 x(0) &= x_o; \quad y(t) =  c^T x(t) 
\end{align}
where $\Bar{\Delta}(x)) = \zeta^T x(t)$, which follows the relation from traditional \textit{adaptive controller} derivation. We refer the reader to \cite{slotine1991applied} for background reading on the theory of State Predictor-based MRAC as it forms the basis of derivation for the theory of \ladap.
Consequently, for a linearly parameterized system in \eqref{eq:er_unct}, we consider the state predictor as
\begin{align} \label{eq:st_prdctr}
\dot{\hat{x}}(t) &= A_m\hat{x}(t) + B_1(\hat{\Delta}(x(t))+\hat{\Bar{\Delta}}(x)+ u_{ad}(t)) + \scriptstyle B_2\dot{\Psi}_{des},\nonumber\\ ~\hat{x}(0) &= x_o; \quad y(t) =  c^T x(t) 
\end{align}
where $\hat{x}(t)\in \mathbb{R}^n$ is the state of the predictor and $\hat{\Delta}(x), \hat{\Bar{\Delta}}(x)= \hat{\zeta}x(t) \in \mathbb{R}^n$ are the estimates of the \textit{unknown} parameter $\Delta(x(t))\in \mathbb{R}^n$, and $\Bar{\Delta}(x)\in \mathbb{R}^n$.

The \emph{total} true uncertainty $(\Delta(x) + \Bar{\Delta}(x))$ is unknown\footnote{Note that dependence on time will only be shown when introducing new concepts or symbols. Subsequent usage will drop these dependencies for clarity of presentation.}, but it is assumed to be continuous over the compact set $D_x  \subset \mathbb{R}^n$. In this paper, we use the \emph{adaptive laws} from the traditional \ladap~to learn part of the total true uncertainty, $\Bar{\Delta}(x)$, and we use the augmented DNN framework to learn the \emph{residual} uncertainty, $\Delta(x)$. We use a DNN framework to represent a function when the basis vector is unknown. Using DNNs, a non-linearly parameterized network estimate of the uncertainty can be written as $\hat{\Delta}(x) \triangleq \theta_n^T \Phi(x)$, where $\theta_n \in \mathbb{R}^{k\times m}$ are network weights for the final layer and $\Phi(x)= \phi_n (\theta_{n-1}, \phi_{n-1}(\theta_{n-2},\phi_{n-2}(\dots))))$, is a $k$ dimensional feature vector which is a function of the inner layer weights, activations and inputs. The basis vector $\Phi(x)\in F:\mathbb{R}^n \rightarrow \mathbb{R}^k$ is considered to be Lipschitz continuous to ensure the existence and uniqueness of the solution. 

\subsection{Adaptive Control Law}
We now derive $u_{ad}$ in \eqref{eq:totl_contrl_law}. We know that the form of $u_m$ satisfies the matching conditions such that $A_m\triangleq A - B_1k_m$, which is a typical form of the \textit{full state feedback controller}. Further, we would like to choose $u_{ad}$ such that it cancels out $\Delta(x) + \Bar{\Delta}(x)$, however, since $\Delta(x)$ and $\Bar{\Delta}(x)$ are unknown, we \emph{estimate} $u_{ad}=\hat{\Delta}(x)+ \hat{\Bar{\Delta}}(x)$. 

From \cite{hovakimyan20111}, we know that \ladap~framework is implemented using a low-pass filter of the form $C(s)$, where $C(s)$ is a BIBO-stable and stricly proper transfer function with DC gain $C(0)=1$, and its state space realization assumes zero initialization. Therefore, the control law term $u_{ad}$ is given by
\begin{align}\label{eq:u_ad_tf}
    u_{ad} (s)=- C(s)(\hat{\eta}(s) + \hat{\Bar{\eta}}(s)- k_g r(s)) 
\end{align}
where $r(s)$, $\hat{\eta}(s)$, $\hat{\Bar{\eta}}(s)$ are the Laplace transforms of reference signal $r(t)$, $\hat{\eta}(t) \triangleq  \hat{\Delta}(x(t))  $, and $\hat{\Bar{\eta}}(t) \triangleq  \hat{\Bar{\Delta}}(x(t))  $ , and the reference signal gain $k_g \triangleq -1/(c^T A_m^{-1}B_1)$. Furthermore, the reference signal $r(t)$ is assumed to be bounded, piece wise continuous, with quantifiable transient and steady-state performance bounds.

In this paper, we are using an \emph{online} DNN framework similar to \cite{joshi2019deep} to estimate \emph{residual} uncertainty, where the true estimates of the labeled pairs of state and true residual uncertainty $\{x(t),\Delta(x(t))\}$ are needed. 
To achieve the asymptotic convergence of the predictive model tracking error to zero, we substitute the DNN estimate in the controller $u_{ad}$, where 
\begin{align}
  \hat{\Delta} = \Delta'(x) = W^T \phi_n (\theta_{n-1}, \phi_{n-1}(\theta_{n-2},\phi_{n-2}(\dots))))
\end{align}

We now state Assumption 1 from \cite{joshi2019deep}.
 \begin{assumption} (from \cite{joshi2019deep})\label{asum:1}
 Appealing to the universal approximation of Neural networks we have that, for every given basis functions $\Phi(x)\in \mathcal{F}$ there exists unique ideal weights $W^* \in \mathbb{R}^{k\times m}$ and $\epsilon_1(x)\mathbb{R}^m$ such that the following approximation holds
 \end{assumption}
 \begin{align}
     \Delta(x) = W^{*^T} \Phi(x) +\epsilon_1(x), ~ \forall x \in D_x \subset \mathbb{R}^n
 \end{align}
 \begin{fact} (from \cite{joshi2019deep})\label{fct:1}
 The network approximation error $\epsilon_1(x)$ is upper bounded, s.t. $\epsilon_1 = \sup_{x\in D_x}\|\epsilon_1(x)\|$, and can be made arbitrarily small given sufficiently large number of basis functions. 
 \end{fact}
 
 Next, we discuss the estimation method of the unknown true network parameters $W^*$. The unknown true network parameters $W^*$ are estimated using the estimation or the adaptive law given as
 \begin{align}\label{eq:adap_law1}
\dot{W}(x) = \Gamma_1 proj (W,\Tilde{x}^T(t)PB_1\Phi(x)),~ W(0) = W_0 \in D_x
\end{align}
where $\Tilde{x}\triangleq \hat{x}-x$ is the prediction error, $\Gamma_1 \in \mathbb{R}_+$ is the adaptation gain, and $P=P^T>0$ solves the algebraic Lyapunov equation $A_m^T P + P A_m= -Q$ for arbitrary symmetric $Q=Q^T>0$. The projection is confined to the set $D_x$. With the neural network-based \emph{adaptive law} stated in \eqref{eq:adap_law1}, the adaptation law from a \ladap~can be represented as
 \begin{align}\label{eq:adap_law2}
\dot{\hat{\zeta}}(x) = \Gamma_2 proj (\hat{\zeta},\Tilde{x}^T(t)PB_1x(t)),~ \hat{\zeta}(0) = \hat{\zeta}_0 \in D_x
\end{align}

Then, taking the Laplace transform of $u_{ad}(t)$, the adaptive control signal looks like
\begin{align}\label{eq:u_ad_tf1}
     u_{ad} (s)=- C(s)(\hat{\eta}(s) + \hat{\Bar{\eta}}(s)- k_g r(s)) 
\end{align}

Now that the \ladap~is defined using the relationships in \eqref{eq:totl_contrl_law}, \eqref{eq:st_prdctr}, \eqref{eq:u_ad_tf} and \eqref{eq:adap_law1}, with $k_m$ and $C(s)$, it must satisfy the $\mathcal{L}_1$-norm condition given by
\begin{align}\label{eq:l1norm}
    \lambda_1 \triangleq \|G(s) \|_{\mathcal{L}_1} L_1 < 1
\end{align}
where
$G(s) \triangleq H(s)(1-C(s)$, $H(s)\triangleq (s\mathbb{I}-A_m)^{-1}B_1$, $L_1\triangleq  \underset{\zeta\in D_x}{\max} \|\zeta\|_1$. However, in this work, since we are using neural networks to learn the \textit{residual} uncertainty parameter, we have an additional $\mathcal{L}_1$ norm condition to satisfy 
\begin{align}\label{eq:l1norm2}
    \lambda_2 \triangleq \|G(s) \|_{\mathcal{L}_1} L_2 < 1
\end{align}
where $L_2\triangleq  \underset{W\in D_x}{\max} \|W\|_1$.
Therefore, we define the following nonadaptive version of the adaptive control system, also known as  the closed-loop reference system of the class of systems in \eqref{eq:er_gnrl_st_sp}
\begin{align} \label{eq:closd_ref}
\dot{x}_{ref}(t) &=  B_1(u_{ref}(t)+ \Delta(x_{ref}(t))+ \bar{\zeta}^Tx_{ref}(t)) \dots \nonumber \\
&+ Ax_{ref}(t) + B_2\dot{\Psi}_{des}, ~~ x_{ref}(0) =x_0 \nonumber \\
 u_{ref} (s)=- &C(s)(\hat{\eta}(s) + \bar{\zeta}^Tx_{ref}(s)- k_g r(s)- k_m x_{ref}(s)) 
\end{align}

Next, we provide the first theoretical result in Lemma~\ref{lem:l1_cond} where we signify the importance of the $\mathcal{L}_1$ condition, such that $x_{ref}(t)$ can be shown to be uniformly bounded.
\begin{lemma}\label{lem:l1_cond}
If $\|G(s)\|_{\mathcal{L}_1}L<1$, then the system in \eqref{eq:closd_ref} is bounded-input bounded-state stable with respect to the reference signal $r(t)$ and $x_0$.
\end{lemma}
\begin{proof}
From the definition of the closed-loop reference system in \eqref{eq:closd_ref} , it follows

\begin{align}
        x_{ref}(s) &= G(s)\eta(s) + G(s)\bar{\zeta}^Tx_{ref}(s) + k_g H(s)C(s)r(s)  \nonumber\\
        &\cdots+ H_1(s) f(s)+ x_{in}(s)
\end{align}
where $x_{in}(s) \triangleq (s\mathbb{I}-A_m)^{-1}x_o$, $\eta(s)$ is the Laplace transform of $\Delta(x_{ref}(t))$. Recalling the fact that $H(s),C(s)$ and $G(s)$ are proper BIBO-stable transfer functions, it follows from \eqref{eq:closd_ref} that for $\tau \in [0,\infty)$, the following bounds hold:
\begin{align}
   &\|x_{ref_{\tau}}\|_{\mathcal{L}_{\infty}} \leq \|G(s)\|_{\mathcal{L}_{1}}\|\eta_{\tau}\|_{\mathcal{L}_{\infty}} + \|G(s)\bar{\zeta}^T\|_{\mathcal{L}_1}\|x_{ref_{\tau}}\|_{\mathcal{L}_{\infty}}   \nonumber\\
        &+ \|k_g H(s)C(s)\|_{\mathcal{L}_{1}}\|r_{\tau}\|_{\mathcal{L}_{\infty}}+ \|H_1(s) f(s)\|_{\mathcal{L}_{1}}+ \|x_{in_{\tau}} \|_{\mathcal{L}_{\infty}}
\end{align}
Since $A_m$ is Hurwitz, $x_{in}(t)$ is uniformly bounded. Thus, using the relation in \eqref{eq:l1norm} we can imply that
\begin{align}\label{eq:bound_1}
\|G(s) \zeta^T\|_{\mathcal{L}_1} = \max_{i=1,\dots, n}\|G_i(s)\|_{\mathcal{L}_1}\sum^n_{j=1}|\zeta_j |\leq \scriptstyle \|G(s) \|_{\mathcal{L}_1}  L_1<1
\end{align}
Consequently,
\begin{align}\label{eq:stabil}
   &\|x_{ref_{\tau}}\|_{\mathcal{L}_{\infty}} \leq \frac{\|G(s)\|_{\mathcal{L}_{1}}\|\eta_{\tau}\|_{\mathcal{L}_{\infty}}}{1-\|G(s)\bar{\zeta}^T\|_{\mathcal{L}_1}} + \frac{\|H_1(s) f(s)\|_{\mathcal{L}_{1}}}{1-\|G(s)\bar{\zeta}^T\|_{\mathcal{L}_1}} \cdots \nonumber\\
        &+ \frac{\|k_g H(s)C(s)\|_{\mathcal{L}_{1}}}{1-\|G(s)\bar{\zeta}^T\|_{\mathcal{L}_1}}\|r_{\tau}\|_{\mathcal{L}_{\infty}}+ \frac{\|x_{in_{\tau}} \|_{\mathcal{L}_{\infty}}}{1-\|G(s)\bar{\zeta}^T\|_{\mathcal{L}_1}}
\end{align}
Since $r(t)$ and $x_{in}(t)$ are uniformly bounded, and \eqref{eq:stabil} holds uniformly for $\tau \in [0,\infty)$, $x_{ref}(t)$ is uniformly bounded. This completes the proof. Note that the \textit{true uncertainty} $\Delta(x_{ref}(t))$ is assumed to be continuous and bounded over a compact domain $D_x \subset \mathbb{R}^n$, which implies $\eta(s)$ is also bounded.
\end{proof}

\subsection{Transient and Steady-State Performance}
For the error $\Tilde{x}=\hat{x}-x$, the error dynamics $\dot{\Tilde{x}}$, using \eqref{eq:er_unct} and \eqref{eq:st_prdctr} can be written as
\begin{align}\label{eq:er_dyn}
    \dot{\Tilde{x}} = A_m \Tilde{x} + B_1( \Tilde{W}^T \Phi(x) + \Tilde{\Bar{\zeta}}x+ \epsilon_1(x))
    \end{align}
    where $\Tilde{W}\triangleq W^* -W$. Further, we let $\Tilde{\eta}(t) \triangleq \Tilde{W}^T\Phi(x(t))$, with $\Tilde{n}(s)$ being its Laplace transform and $\Tilde{\Bar{\eta}}(t) \triangleq \Tilde{\Bar{\zeta}}(t)x(t)$ with $\Tilde{\Bar{\eta}}(s)$ being its Laplace transform. Therefore, the error dynamics in frequency domain can be re-written as 
    \begin{align}\label{eq:er_dyn_fs}
    \Tilde{x}(s) = H(s)(\Tilde{\eta}(s) + \Tilde{\Bar{\eta}}(s) + \epsilon_1(x(s)))
    \end{align}
    where $\epsilon_1(x(s))$ is the Laplace transform of $\epsilon_1(x(t))$.
    Before we present the uniform boundedness result of $\Tilde{x}$ in Lemma~\ref{lem:pred_er}, we state Assumption~\ref{asum:2}
    \begin{assumption} (from \cite{joshi2019deep})\label{asum:2}
    For uncertainty parameterized by unknown true weights $W^* \in \mathbb{R}^{k\times n}$ and known nonlinear $\phi(x)$, the ideal weight matrix is assumed to be upper bounded s.t. $\|W^*\|\leq W_b$.
    \end{assumption}
    As done in \cite{joshi2019deep}, we will prove the following lemma under switching feature vector assumption.
    \begin{lemma}\label{lem:pred_er}
    The prediction error $\Tilde{x}$  is uniformly bounded.
    \end{lemma}
    \begin{proof}
   The feature vectors belong to a function class
characterized by the inner layer network weights $\theta_i$ s.t $\Phi \in \mathcal{F}$. Using the same manipulation process in Theorem~1 in \cite{joshi2019deep}, we will prove the Lyapunov stability under the assumption that inner layer of DNN presents us a feature which results in the worst possible approximation error compared to network with features before switch. 

For this purpose, we will denote $\Phi(x)$ as feature before switch and $\Bar{\Phi}(x)$ as the feature after switch. We define the error $\epsilon_2(x)$ as
\begin{align}
    \epsilon_2 (x) = \underset{\Bar{\Phi}\in F}{\sup} |W^T\Bar{\phi}(x) - W^T\Phi (x)|
\end{align}
Using Fact~\ref{fct:1}, we can upper bound $\epsilon_2(x)$ as $\Bar{\epsilon}_2=\underset{x\in D_x}{\sup} \|\epsilon_2 (x)\|$. Thus, by adding and subtracting the term $W^T \Bar{\Phi}(x)$, we can rewrite the error dynamics \eqref{eq:er_dyn}
\begin{align}\label{eq:er_dyn2}
    \dot{\Tilde{x}} &= A_m \Tilde{x} + B_1( W^{*T}\Phi (x) - W^T \Phi(x) +  W^T \Bar{\Phi}(x) \cdots \nonumber\\
    &- W^T \Bar{\Phi}(x)+ \Tilde{\Bar{\zeta}}x(t))+ \epsilon_1(x)
    \end{align}
    Then by applying Assumption~\ref{asum:1}, there exists $W^* \forall \Phi \in \mathcal{F}$. Therefore, we can replace $W^{*T}\Phi(x)$ by $W^{*T}\Bar{\Phi}(x)$ in \eqref{eq:er_dyn} to obtain
    \begin{align}\label{eq:er_dyn3}
    \dot{\Tilde{x}} = A_m \Tilde{x} + B_1( \Tilde{W}^{*T}\Bar{\Phi} (x) + \Tilde{\Bar{\zeta}}x(t))+ \epsilon_1(x) + \epsilon_2 (x)
    \end{align}

Thus, we consider the following Lyapunov function
    
    \begin{align}\label{eq:lyapunov_1}
    V(\Tilde{x},\Tilde{W},\Tilde{\Bar{\zeta}}) = \Tilde{x}^T P \Tilde{x} + \frac{\Tilde{W}^T \Gamma_1^{-1}\Tilde{W}}{2} + \frac{\Tilde{\Bar{\zeta}}^T \Tilde{\Bar{\zeta}}}{\Gamma_2}
    \end{align}
   Now taking the time derivative of the Lyapunov function in \eqref{eq:lyapunov_1}, we get $\dot{V}$ as
     \begin{align}\label{eq:der_lyapunov_1}
    \dot{V}(\Tilde{x},\Tilde{W},\Tilde{\Bar{\zeta}}) \leq -\Tilde{x}^T Q \Tilde{x} + 2\Tilde{x}^T P \epsilon(x)
    \end{align}
    where $\epsilon(x) = \epsilon_1(x) + \epsilon_2 (x)$ and $\bar{\epsilon} = \Bar{\epsilon_1} + \Bar{\epsilon_2}$.
    Hence, Lyapunov stability $\dot{V}(\Tilde{x},\Tilde{W},\Tilde{\Bar{\zeta}})\leq0$ is obtained outside compact neighborhood of the origin $\Tilde{x}=0$, for some sufficiently large $\lambda_{\min}(Q)$ where we can show a lower bound on $\Tilde{x}$ as
    \begin{align}\label{eq:er_low_bnd}
      \|\Tilde{x}(t)\|\geq \frac{2 \lambda_{\max}(P)\Bar{\epsilon}}{\lambda_{\min}(Q)} = LB  
    \end{align}
    Now, what remains is to show that $\Tilde{x}(t)$ is upper bounded.
    In this direction, since $\Tilde{x}(0)=0$ and in the vicinity of the compact neighborhood of the origin $\Tilde{x}(t)=0$, it follows that
    \begin{align}
        \lambda_{\min}(P) \|\Tilde{x}(t) \|^2 &\leq V(t) \leq V(0) =  \frac{\Tilde{W}(0)^T \Gamma_1^{-1}\Tilde{W}(0)}{2} \cdots \nonumber\\
        &+ \frac{\Tilde{\Bar{\zeta}}^T(x(0)) \Tilde{\Bar{\zeta}}(x(0))}{\Gamma_2}
    \end{align}
    The projection operators\footnote{Definition of projection operators can be found in \cite{hovakimyan20111}.} with Assumption~\ref{asum:2} ensure that the error parameters $\Tilde{W}$ and $\Tilde{\bar{\zeta}}$ are bounded and so we have
    \begin{align}
    \frac{\Tilde{W}(0)^T \Gamma_1^{-1}\Tilde{W}(0)}{2} \leq \frac{4 \underset{W\in D_x}{\max}\|W\|^2}{\Gamma_1}, \cdots \nonumber \\
    \frac{\Tilde{\Bar{\zeta}}^T(x(0)) \Tilde{\Bar{\zeta}}(x(0))}{\Gamma_2} \leq \frac{4 \underset{\zeta\in D_x}{\max}\|\zeta\|^2}{\Gamma_2}
    \end{align}
    which leads to the following upper bound:
    \begin{align}\label{eq:newprdt_tf}
        \|\Tilde{x}\|^2\leq \frac{W_{\max}}{\lambda_{\min}(P)\Gamma_1} + \frac{\zeta_{\max}}{\lambda_{\min}(P)\Gamma_2}
    \end{align}
    Since $\|\cdot\|_{\infty}\leq \|\cdot\|$, and this bound is uniform, the bound above yields
     \begin{align}\label{eq:bound}
        \|\Tilde{x}_{\tau}\|_{\mathcal{L}_{\infty}}\leq \sqrt{ \frac{W_{\max}}{\lambda_{\min}(P)\Gamma_1} + \frac{\zeta_{\max}}{\lambda_{\min}(P)\Gamma_2}}
    \end{align}
    which holds for every $\tau\geq 0$
    
    \end{proof}
    
    Notice that the bounds on $\Tilde{x}$ are obtained independent of the control term $u_{ad}$. This implies that both $x$ and $\hat{x}$ can diverge at the same rate, maintaining a uniformly bounded error between the two. Therefore, the result in Lemma~\ref{lem:conv} will prove that using the control law $u_{ad}$, the state of the predictor remains bounded and consequently leads to asymptotic convergence of the tracking error $\Tilde{x}$ to the \emph{lower bound} $LB$.
    \begin{lemma}\label{lem:conv}
    For adaptive control law given by $u_{ad}$ and if the conditions in \eqref{eq:l1norm} and \eqref{eq:l1norm2}  hold, then we have the asymptotic convergence $\underset{t\rightarrow \infty}{\lim} \Tilde{x} =LB$.
    \end{lemma}
    \begin{proof}
    To prove asymptotic convergence of $\Tilde{x}(t)$ to $LB$, one needs to ensure that $\hat{x}(t)$ with $u_{ad}(t)$ is uniformly bounded. Therefore, we have
    \begin{align}
        \hat{x}(s) &= G(s)\hat{\eta}(s) + G(s)\hat{\bar{\eta}}(s) + k_g H(s)C(s)r(s) \cdots \nonumber\\
        &+ H_1(s) f(s)+ x_{in}(s)
    \end{align}
    where $H_1(s) \triangleq (s\mathbb{I} - A_m)^{-1} B_2$ and $f(s)= \mathcal{L}(\dot{\Psi}_{des})$ is the Laplace transform, which leads to the following upper bound
    \begin{align}\label{eq:prdict_tf}
        &\|\hat{x}_{\tau}\|_{\mathcal{L}_{\infty}} = \|G(s)\|_{\mathcal{L}_{1}}\|\hat{\eta}_{\tau}\|_{\mathcal{L}_{\infty}} + \|G(s)\|_{\mathcal{L}_{1}}\|\hat{\bar{\eta}}_{\tau}\|_{\mathcal{L}_{\infty}} \cdots \nonumber\\
        &+ \|k_g H(s)C(s)\|_{\mathcal{L}_{1}}\|r_{\tau}\|_{\mathcal{L}_{\infty}} + \| H(s)f(s)\|_{\mathcal{L}_{1}}+ \|x_{in_{\tau}}\|_{\mathcal{L}_{\infty}}
    \end{align}
Next, applying the triangular relationship for norms to the bound in \eqref{eq:bound}, we have
\begin{align}
    |\| \hat{x}_{\tau}\|_{\mathcal{L}_{\infty}} -  \| x_{\tau}\|_{\mathcal{L}_{\infty}} |  \leq         \sqrt{ \frac{W_{\max}}{\lambda_{\min}(P)\Gamma_1} + \frac{\zeta_{\max}}{\lambda_{\min}(P)\Gamma_2}}
\end{align}
        The projection operators ensure that the error parameters $\Tilde{W}$ and $\Tilde{\bar{\zeta}}$ are bounded , and hence we have $\|\hat{\Bar{\eta}}_{\tau}\|_{\mathcal{L}_{\infty}}\leq L_1 \| x_{\tau}\|_{\mathcal{L}_{\infty}}$. 
        Substituting for $\| x_{\tau}\|_{\mathcal{L}_{\infty}}$ yields
    \begin{align}\label{eq:bound3}
        \|\hat{\Bar{\eta}}_{\tau}\|_{\mathcal{L}_{\infty}} \leq L_1 \left (  \| \hat{x}_{\tau}\|_{\mathcal{L}_{\infty}}  +    \sqrt{ \frac{W_{\max}}{\lambda_{\min}(P)\Gamma_1} + \frac{\zeta_{\max}}{\lambda_{\min}(P)\Gamma_2}}             \right )
    \end{align}
    Then the bound on $\| \hat{x}_{\tau}\|_{\mathcal{L}_{\infty}}$ and $ \|\hat{\Bar{\eta}}_{\tau}\|_{\mathcal{L}_{\infty}}$ in \eqref{eq:prdict_tf} and \eqref{eq:bound3}, with account of the stability conditions in \eqref{eq:l1norm} and \eqref{eq:l1norm2}, lead to
    \begin{align}
        &\|\hat{x}_{\tau} \|_{\mathcal{L}_{\infty}} \leq \cdots \nonumber\\
        &\frac{\lambda_1 \sqrt{ \frac{W_{\max}}{\lambda_{\min}(P)\Gamma_1} + \frac{\zeta_{\max}}{\lambda_{\min}(P)\Gamma_2}} + \|G(s)\|_{\mathcal{L}_{1}}\|\hat{\eta}_{\tau}\|_{\mathcal{L}_{\infty}}+\dots   }{1-\lambda_1} \nonumber \\
        &\frac{\|k_g H(s) C(s \|_{\mathcal{L}_{1}} \| r_{\tau} \|_{_{\mathcal{L}_{\infty}}} + \| H(s)f(s)\|_{\mathcal{L}_{1}}+ \|x_{in_{\tau}} \|_{_{\mathcal{L}_{\infty}}}}{1-\lambda_1}
    \end{align}
    Since the bound on the right hand side is uniform, $\hat{x}(t)$ is uniformly bounded. Application of Barbalat's lemma leads to the asymptotic result $\underset{t \rightarrow \infty}{\lim}\Tilde{x}(t)=LB$. Note that $\hat{\eta}_{\tau}$ is bounded because $\Tilde{x}_{\tau}$ is shown to be bounded from Lemma~\ref{lem:pred_er} and $\eta(s)$ is bounded due to $\Delta(x) \subset D_x$.

    \end{proof}
Next, we present the main result of this paper in Theorem~\ref{thrm:1}, which provides the proof of $x$ approaching $x_{ref}$ and $u$ approaching $u_{ref}$, with an upper bound.
    \begin{theorem} \label{thrm:1}
    For the system in \eqref{eq:er_gnrl_st_sp} and the controller defined in \eqref{eq:totl_contrl_law} with the state predictor dynamics given by \eqref{eq:st_prdctr} and the adaptive control law in \eqref{eq:u_ad_tf}, subject to the $\mathcal{L}_1$-norm condition in \eqref{eq:l1norm}, we have
    \noindent\begin{minipage}{0.5\linewidth}
\begin{equation}
 \|x_{ref} - x\|_{\mathcal{L}_{\infty}} \leq \gamma_1 \label{eq:v1}
\end{equation}
\end{minipage}%
\begin{minipage}{0.5\linewidth}
\begin{equation}
  \|u_{ref} - u\|_{\mathcal{L}_{\infty}} \leq \gamma_2 \label{eq:v2}
\end{equation}
\end{minipage}\par\vspace{\belowdisplayskip}
where $\gamma_1$ and $\gamma_2$ are given as
\begin{align}\label{eq:gamma1}
    \gamma_1 = \frac{\|C(s)\|_{\mathcal{L}_{1}}}{1-\|G(s) \|_{\mathcal{L}_{1}}L_1} \sqrt{ \frac{W_{\max}}{\lambda_{\min}(P)\Gamma_1} + \frac{\zeta_{\max}}{\lambda_{\min}(P)\Gamma_2}} \dots \nonumber \\
      + \frac{\|G(s)\|_{\mathcal{L}_{1}}}{1-\|G(s) \|_{\mathcal{L}_{1}}L_1}\|(\Delta(x_{ref})-\Delta(x))_{\tau}\|_{\mathcal{L}_{\infty}} 
\end{align}
\begin{align}\label{eq:gamma2}
    \gamma_2 = \|(C(s) \bar{\zeta}^T+ k_m^T)\|_{\mathcal{L}_{1}} \|(x_{ref}-x)_{\tau}\|_{\mathcal{L}_{\infty}} \dots \nonumber \\
       - \|C(s)\|_{\mathcal{L}_{1}}\|\Tilde{\eta}_{\tau}\|_{\mathcal{L}_{\infty}}- \|C(s)\|_{\mathcal{L}_{1}}\|\epsilon_1(x_{\tau})\|_{\mathcal{L}_{\infty}}\dots \nonumber \\
       +\|C(s)H(s)^{-1}\|_{\mathcal{L}_{1}}\|\Tilde{x}_{\tau}\|_{\mathcal{L}_{\infty}}      
\end{align}
    \end{theorem}
    \begin{proof}
    The response of the closed-loop system in \eqref{eq:er_unct} with the \ladap~in \eqref{eq:u_ad_tf} can be written in the frequency domain as
    \begin{align}
        &x(s) = G(s)\Delta(x(s)) + G(s)\zeta^T x(s) + k_g H(s)C(s)r(s) \cdots \nonumber\\
        &+  H_1(s) f(s)-   H(s)C(s)(\Tilde{\eta}(s) + \Tilde{\Bar{\eta}}(s) +  \epsilon_1(x(s)))+ x_{in}(s)
    \end{align}
    the close-loop reference system is then
      \begin{align}
      x_{ref}(s) &= H(s)k_g C(s)r(s)+ G(s)\zeta^T x_{ref}(s) + x_{in}(s) \dots \nonumber \\
&+ H_1(s) f(s) + G(s)\Delta(x_{ref}(s))
    \end{align}
    then we have 
    \begin{align}\label{eq:ref_er}
      x_{ref}(s)-x(s)  &= G(s)\zeta^T (x_{ref}(s) - x(s)) \dots \nonumber \\
       &+ G(s)(\Delta(x_{ref}(s))-\Delta(x(s))) + C(s)\Tilde{x}(s)
    \end{align}
    which implies that
     \begin{align}
      &\|(x_{ref}-x)_{\tau} \|_{\mathcal{L}_{\infty}} \leq \|G(s)\zeta^T \|_{\mathcal{L}_{1}} \|(x_{ref}(s) - x(s))_{\tau} \|_{\mathcal{L}_{\infty}}\scriptstyle \dots \nonumber \\
       &+ \|C(s)\|_{\mathcal{L}_{1}}\|\Tilde{x}_{\tau}\|_{\mathcal{L}_{\infty}} + \|G(s)\|_{\mathcal{L}_{1}}\|(\Delta(x_{ref})-\Delta(x))_{\tau}\|_{\mathcal{L}_{\infty}}
    \end{align}
    Then the bounds in \eqref{eq:bound_1} and \eqref{eq:bound} lead to the uniform upper bound
      \begin{align}\label{eq:gamma_1}
      &\|(x_{ref}-x)_{\tau} \|_{\mathcal{L}_{\infty}} \leq  \frac{\|C(s)\|_{\mathcal{L}_{1}}}{1-\|G(s) \|_{\mathcal{L}_{1}}L_1}\|\Tilde{x}_{\tau}\|_{\mathcal{L}_{\infty}}\dots \nonumber \\
       & + \frac{\|G(s)\|_{\mathcal{L}_{1}}}{1-\|G(s) \|_{\mathcal{L}_{1}}L_1}\|(\Delta(x_{ref})-\Delta(x))_{\tau}\|_{\mathcal{L}_{\infty}}\dots \nonumber \\
       &  \leq \gamma_1 = \frac{\|C(s)\|_{\mathcal{L}_{1}}}{1-\|G(s) \|_{\mathcal{L}_{1}}L_1} \sqrt{ \frac{W_{\max}}{\lambda_{\min}(P)\Gamma_1} + \frac{\zeta_{\max}}{\lambda_{\min}(P)\Gamma_2}} \dots \nonumber \\
       & + \frac{\|G(s)\|_{\mathcal{L}_{1}}}{1-\|G(s) \|_{\mathcal{L}_{1}}L_1}\|(\Delta(x_{ref})-\Delta(x))_{\tau}\|_{\mathcal{L}_{\infty}} 
    \end{align}
  Note that by invoking the idea of \textit{projection} operator we can claim boundedness on the term $\|(\Delta(x_{ref})-\Delta(x))_{\tau}\|_{\mathcal{L}_{\infty}}$. Further, using Lemma~\ref{lem:conv}, with Assumption~\ref{asum:1}, we can say that in the vicinity of $\Tilde{x}(t)\rightarrow 0$, $\Delta(x)\rightarrow \Delta(x_{ref})$.
    
    
    Next, we will derive the second bound in \eqref{eq:v2}, for which we have 
    \begin{align}\label{eq:law_er}
       & u_{ref}(s) - u(s) = -(C(s) \bar{\zeta}^T+ k_m^T) (x_{ref}(s)-x(s)) \dots \nonumber \\
       &+C(s)\Tilde{\Bar{\eta}}(s)-C(s)(\Delta(x_{ref}(s))-\Delta(x(s)))
    \end{align}
    which, again by using Lemma~\ref{lem:conv} with Assumption~\ref{asum:1}, simplifies to 
    \begin{align}
       & u_{ref}(s) - u(s) = -(C(s) \bar{\zeta}^T+ k_m^T) (x_{ref}(s)-x(s)) \dots \nonumber \\
       &+C(s)(H(s)^{-1}\Tilde{x}(s)- \Tilde{\eta}(s)- \epsilon_1(x(s)))
    \end{align}
    and consequently, we have:
    \begin{align}\label{eq:gamma_2}
       & \|(u_{ref} - u)_{\tau}\|_{\mathcal{L}_{\infty}} \leq \gamma_2 = \dots \nonumber \\
       &\|(C(s) \bar{\zeta}^T+ k_m^T)\|_{\mathcal{L}_{1}} \|(x_{ref}-x)_{\tau}\|_{\mathcal{L}_{\infty}} \dots \nonumber \\
       &+\|C(s)H(s)^{-1}\|_{\mathcal{L}_{1}}\|\Tilde{x}_{\tau}\|_{\mathcal{L}_{\infty}}\dots \nonumber \\
       &- \|C(s)\|_{\mathcal{L}_{1}}\|\Tilde{\eta}_{\tau}\|_{\mathcal{L}_{\infty}}- \|C(s)\|_{\mathcal{L}_{1}}\|\epsilon_1(x_{\tau})\|_{\mathcal{L}_{\infty}}
    \end{align}
    Hence, the above relation can be claimed to be bounded. 

    \end{proof}

     \subsection{Stability Analysis for Controller with DNN}
    The generalization error of a machine learning model is defined as the difference between the empirical loss of the training set and the expected loss of test set. This property determines the ability of the trained model to generalize well from the learning data to new unseen data. The capability of the model to extrapolate from training data to new data is determined. The \textit{generalization} error is typically defined as
    \begin{align}
        \hat{\Delta}(x) - f_{\theta}(x) \leq \epsilon_G
    \end{align}
    where $f_{\theta}$ is the learned approximation to the model uncertainty with parameters $\theta \in \Theta$, where $\Theta$ is the space of parameters, i.e. $f_{\theta}: \mathbb{R}^n \rightarrow \mathbb{R}^m$.
    Using the deep learning framework with controller \eqref{eq:totl_contrl_law} and using the system in \eqref{eq:er_gnrl_st_sp}, we can re-write the system error dynamics as
    \begin{align}\label{eq:er_dyn2}
    \dot{\Tilde{x}} = A_m \Tilde{x} + B_1( f_{\theta}(x) - \Delta(x) + \Tilde{\Bar{\Delta}}(x))
    \end{align}
    Adding and subtracting the term $\Delta'(x)$ in above expression and using the training and generalization error definitions we can write the above as
    \begin{align}\label{eq:er_dyn3}
    \dot{\Tilde{x}} = A_m \Tilde{x} + B_1( f_{\theta}(x)  -\Delta'(x)  + \Delta'(x) -\Delta(x) + \Tilde{\Bar{\Delta}}(x))
    \end{align}
    The term $(\Delta'(x)-\Delta(x))$ is the DNN training error and for simplicity we assume the training error to be zero as also done in \cite{joshi2019deep}. The term $(f_{\theta}(x)  -\Delta'(x))$ is the generalization error of our machine learning model. Therefore, the final predictor error dynamics looks like  \begin{align}\label{eq:er_dyn4}
    \dot{\Tilde{x}} = A_m \Tilde{x} + B_1( \epsilon_G + \Tilde{\Bar{\Delta}}(x))
    \end{align}
    Now, the asymptotic tracking performance of the error dynamics under the deep learning framework can be determined by first defining a Lyapunov candidate function of the type 
    \begin{align}
        V(\Tilde{x},\Tilde{\Bar{\Delta}})=\Tilde{x}^T P \Tilde{x} + \frac{\Tilde{\Bar{\Delta}}^T \Tilde{\Bar{\Delta}}}{\Gamma_2}
    \end{align}
    Taking the time derivative of the above Lyapunov function we get
    \begin{align}
        \dot{V}(\Tilde{x},\Tilde{\Bar{\Delta}})=-\Tilde{x}^T Q \Tilde{x} + 2 \epsilon_G B_1 P \Tilde{x} 
    \end{align}
    where $Q$ is the solution of Lyapunov equation $A_m^T P + P A_m= -Q$. To satisfy the condition $\dot{V}(\Tilde{x},\Tilde{\Bar{\Delta}})\leq 0$ we get the following upper bounds on the generalization error
    \begin{align}\label{eq:DNN_bound}
        \|\epsilon_G\|< \frac{\lambda_{\max}(Q)\|\Tilde{x}\|}{\lambda_{\min}(P)}
    \end{align}
    The main idea here is that the DNN framework produces a generalization error lower than the bound in \eqref{eq:DNN_bound} which then guarantees Lyapunov stability of the system under \neural. 
Further, using the above Lyapunov analysis, it is important to show information theoretically the lower bound on the number of independent samples that is needed to train, before the DNN generalization error can be claimed to be below a determined lower level as derived in \eqref{eq:DNN_bound}. Since our Lyapunov analysis gives an upper bound in \eqref{eq:DNN_bound} that is similar to the upper bound in \cite{joshi2019deep}, we refer the reader to Theorem 2 in \cite{joshi2019deep} where the authors study the sample complexity results from computational theory and show that when applied to a network learning real-valued functions the number of training samples grows at least linearly with the number of tunable parameters to achieve the generalization error specified in \eqref{eq:DNN_bound}.

\section{Simulation Experiments}\label{SIM-Results}

In this section we discuss the simulation setup and provide simulation results using PyBullet \cite{coumans2021} for F1TENTH. We evaluate our proposed \neural~against the \LF~\cite{rajamani2011vehicle}, \ladap~\cite{hovakimyan20111} and the \deepM~ \cite{joshi2019deep} controllers. We do not provide any evaluation of the adopted DNN framework from \cite{joshi2019deep}. Instead, the
reader is referred to \cite{joshi2019deep} for complete evaluation of the DNN training, testing and
validation performances.

\subsection{Simulation Setup}\label{subsec-sim_setup}
Open source PyBullet physics simulator is used to simulate the F1TENTH vehicle based on the URDF model design presented by Babu and Behl \cite{babu2020f1tenth}. The robot model was slightly modified to account for the parameters identified in our previous work \cite{gonultas2023system} and PyBullet is used as a direct replacement of the authors' physics engine of choice.

\subsection{Simulation Results}\label{subsec-sim_rslts}





\begin{figure*}[!h]
     \centering
     \begin{subfigure}[b]{0.24\textwidth}
         \centering
         \includegraphics[width=\textwidth]{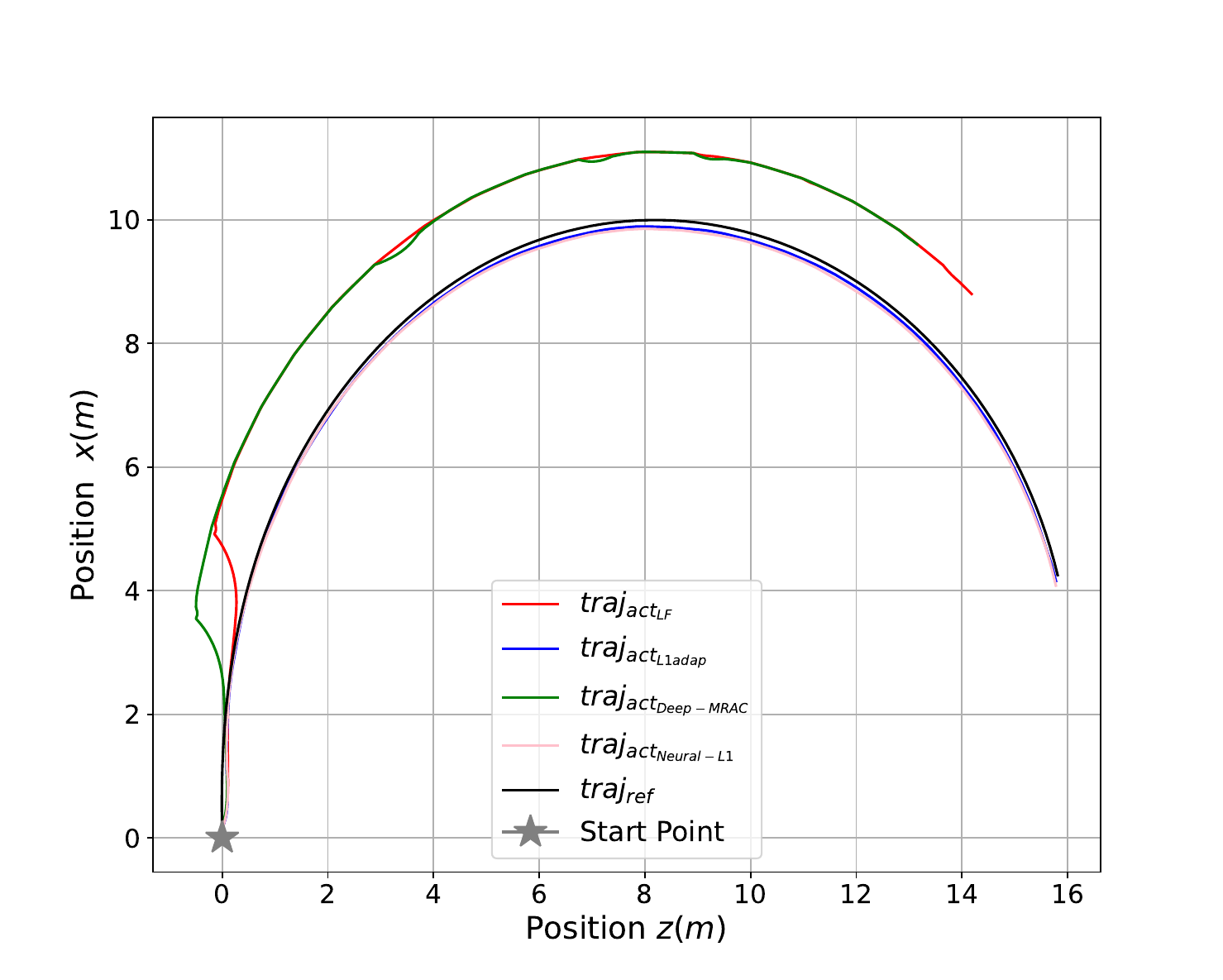}
         \caption{Trajectory tracking  performance of \ladap~and \neural~controllers are superior to \LF~and~\deepM~controllers.}
         \label{fig:circ_traj}
     \end{subfigure}
     \hfill
     \begin{subfigure}[b]{0.24\textwidth}
         \centering
         \includegraphics[width=\textwidth]{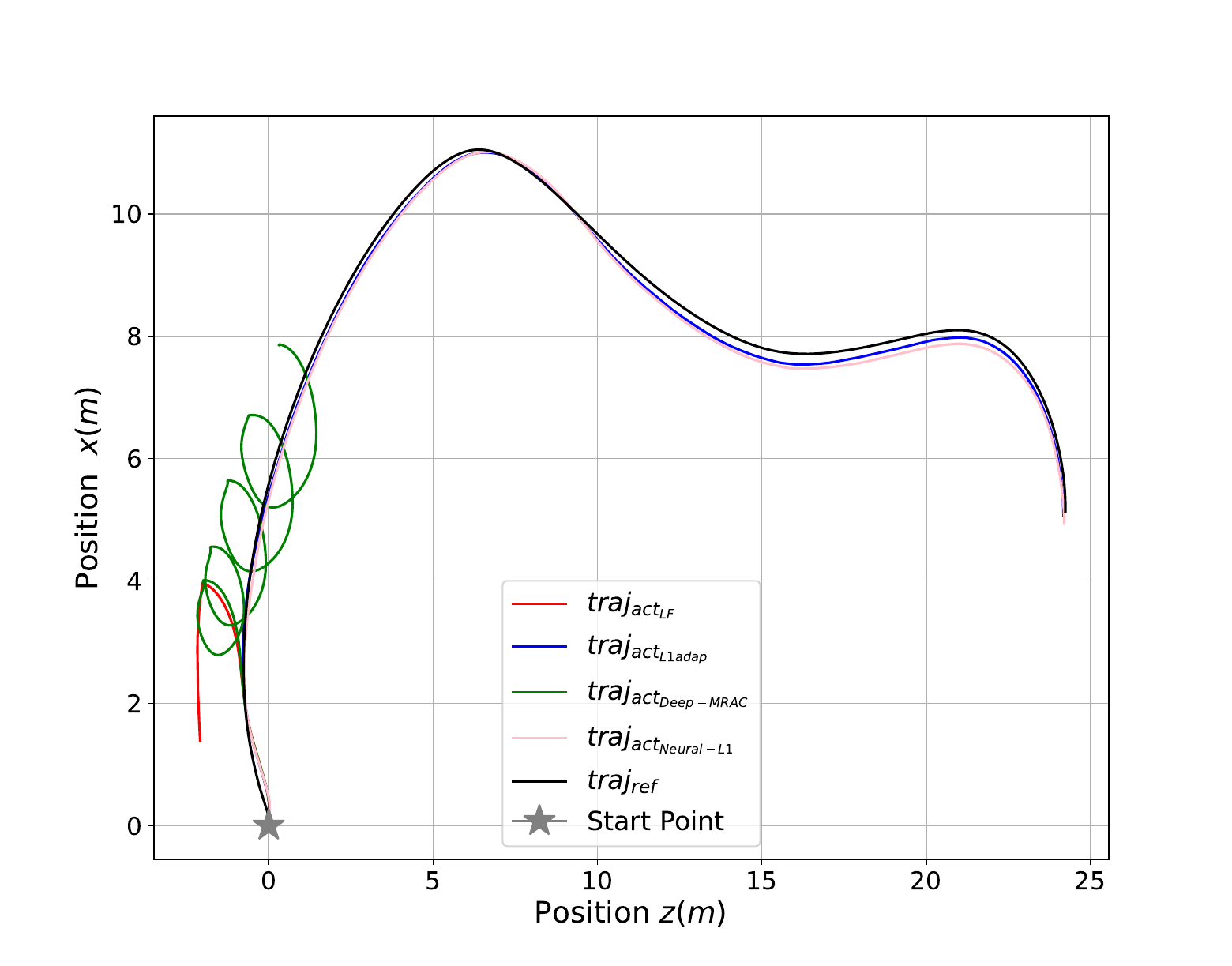}
         \caption{Over this arbitrary (varying radius $R$) trajectory, \LF~and \deepM~controllers fail to track the \textit{ref} trajectory. }
         \label{fig:berlin_traj}
     \end{subfigure}
     \hfill
     \begin{subfigure}[b]{0.24\textwidth}
         \centering
         \includegraphics[width=\textwidth]{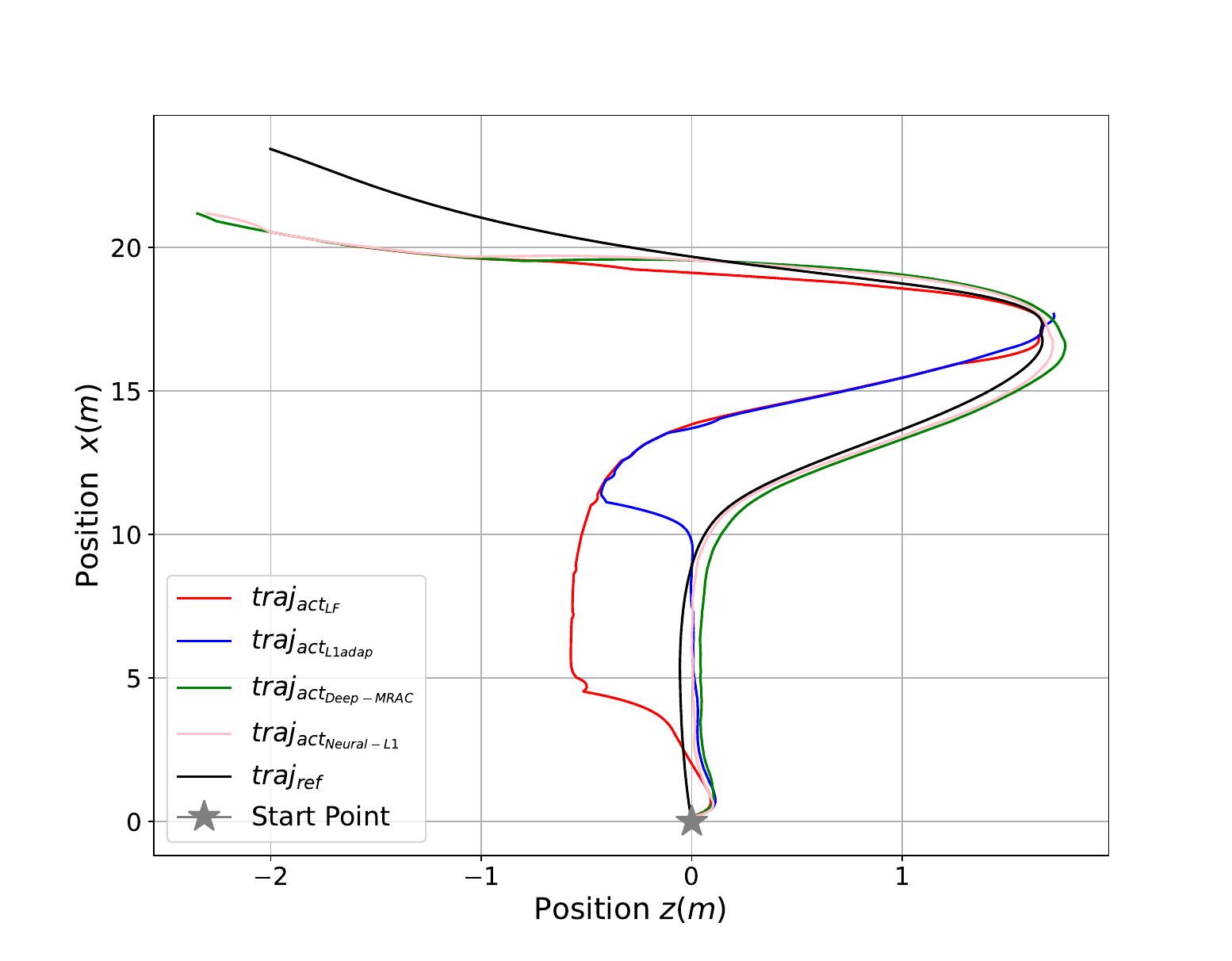}
         \caption{Over this arbitrary (varying radius $R$) trajectory, \LF~and \ladap~controllers fail to track the \textit{ref} trajectory.}
         \label{fig:gbr_traj}
     \end{subfigure}
     \hfill
     \begin{subfigure}[b]{0.24\textwidth}
         \centering
         \includegraphics[width=\textwidth]{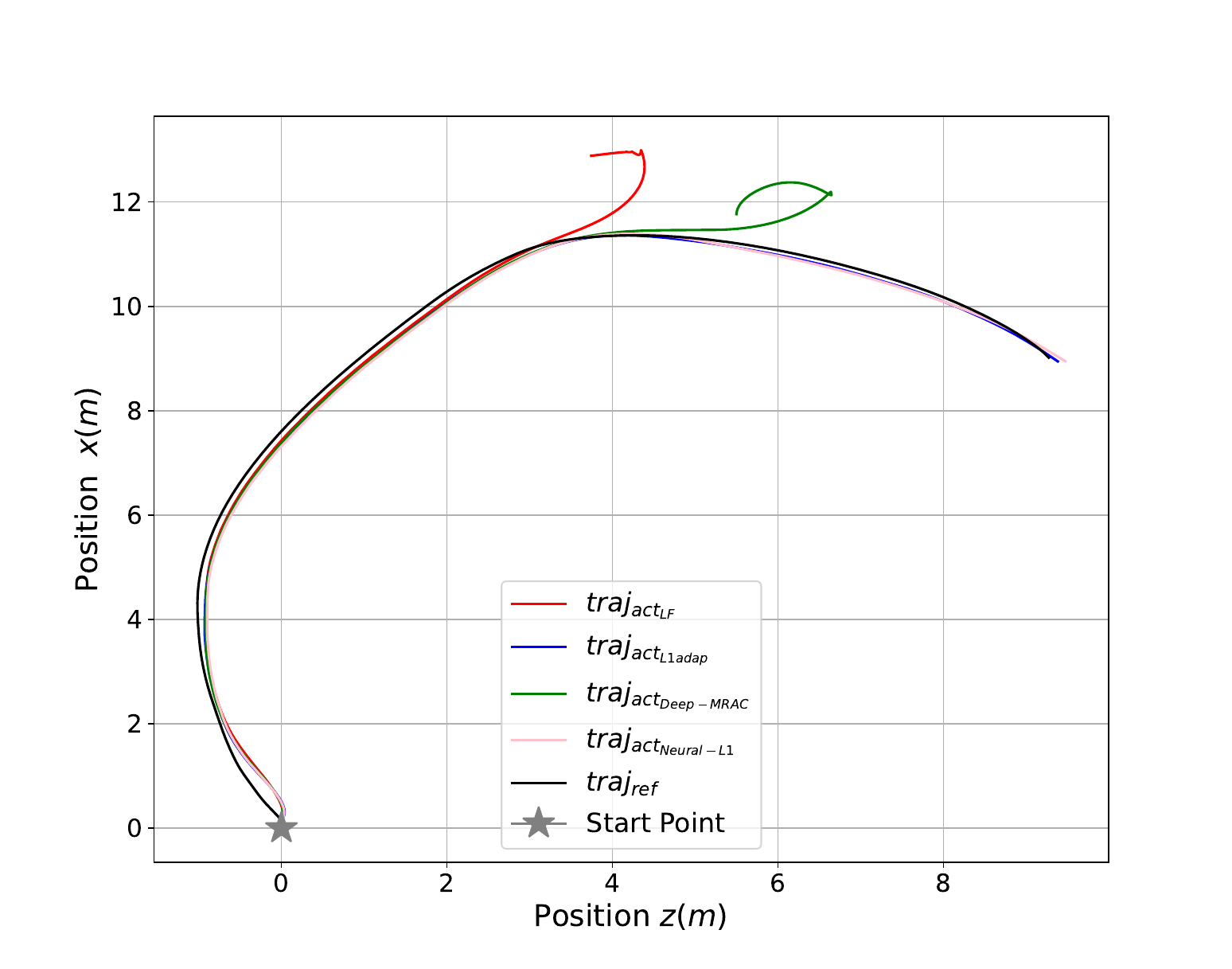}
         \caption{Over this arbitrary (varying radius $R$) trajectory, \LF~and \deepM~controllers fail to track the \textit{ref} trajectory.}
         \label{fig:montreal_traj}
     \end{subfigure}
        \caption{PyBullet simulation of lane-keeping dynamics comparing trajectory tracking performance on \textit{arbitrary} trajectories for our proposed controller \neural(pink) against existing controllers: \LF(red), \ladap(blue) and \deepM(green).}
        \label{fig:sim_plots}
\end{figure*}

Trajectory tracking performance of all controllers, over arbitrary trajectories, is shown in Fig.~\ref{fig:sim_plots}. Over a \textit{circular} reference (\textit{ref}) trajectory in Fig.~\ref{fig:circ_traj} with a constant radius $R=10m$ and vehicle longitudinal velocity $V_x = 10m/s$, \neural~and \ladap~exhibit a superior tracking performance in comparison to \LF~\cite{hovakimyan20111} and \deepM~\cite{joshi2019deep} controllers. Uncertainty, as plotted in \emph{black} in Fig.~\ref{fig:uncertainty_tracking}, is introduced directly into the control law $u(t)$ as a function of constant unknown parameters $[0.5314, 0.16918, -0.6245, 0.1095]$ with added uniform white noise $[-0.1,0.1]$. Note that the constant unknown parameters and the uniform white noise add up to the total \emph{true} uncertainty $\Delta(x)+\Bar{\Delta}(x)$.
In figures \ref{fig:berlin_traj}, \ref{fig:gbr_traj} and \ref{fig:montreal_traj} , we further provide the tracking performance results of all controllers over three other \emph{arbitrary} trajectories with varying radius $R$, even though the lateral error dynamics \cite{rajamani2011vehicle} is based on the assumption that the trajectory radius remains constant.
\neural~exhibits superior tracking performance in comparison to the other controllers as it successfully completes each of the four trajectories in figures \ref{fig:circ_traj}, \ref{fig:berlin_traj}, \ref{fig:gbr_traj} and \ref{fig:montreal_traj} , whereas \LF, \ladap~and \deepM~controllers fail to track at least one of the four \textit{ref} trajectories.

Next, in figures \ref{fig:state_tracking} and \ref{fig:uncertainty_tracking}, we study the profiles of the \emph{states} $[e_1, \dot{e}_1, e_2, \dot{e}_2]$ of the lateral error dynamics from \cite{rajamani2011vehicle} and the \emph{learned} uncertainty of the adaptive controllers, only corresponding to the circular trajectory shown in Fig. \ref{fig:circ_traj}.
Thus, in Fig. \ref{fig:state_tracking}, the states \emph{lateral error} $e_1$ and \emph{yaw angle error} $e_2$ deviate from the reference for \LF~controller and the \deepM~controller, implying unstable behavior of the controllers. Only for the controllers \neural~and \ladap~all four states converge \emph{close} to zero. This implies that both \neural~and \ladap~exhibit stability in the presence of disturbance directly induced into the control signal $u(t)$ which is also the steering angle command $\delta(t)$ in radians.
\begin{figure}
\centering
\subfloat{\includegraphics[width = 0.52\textwidth]{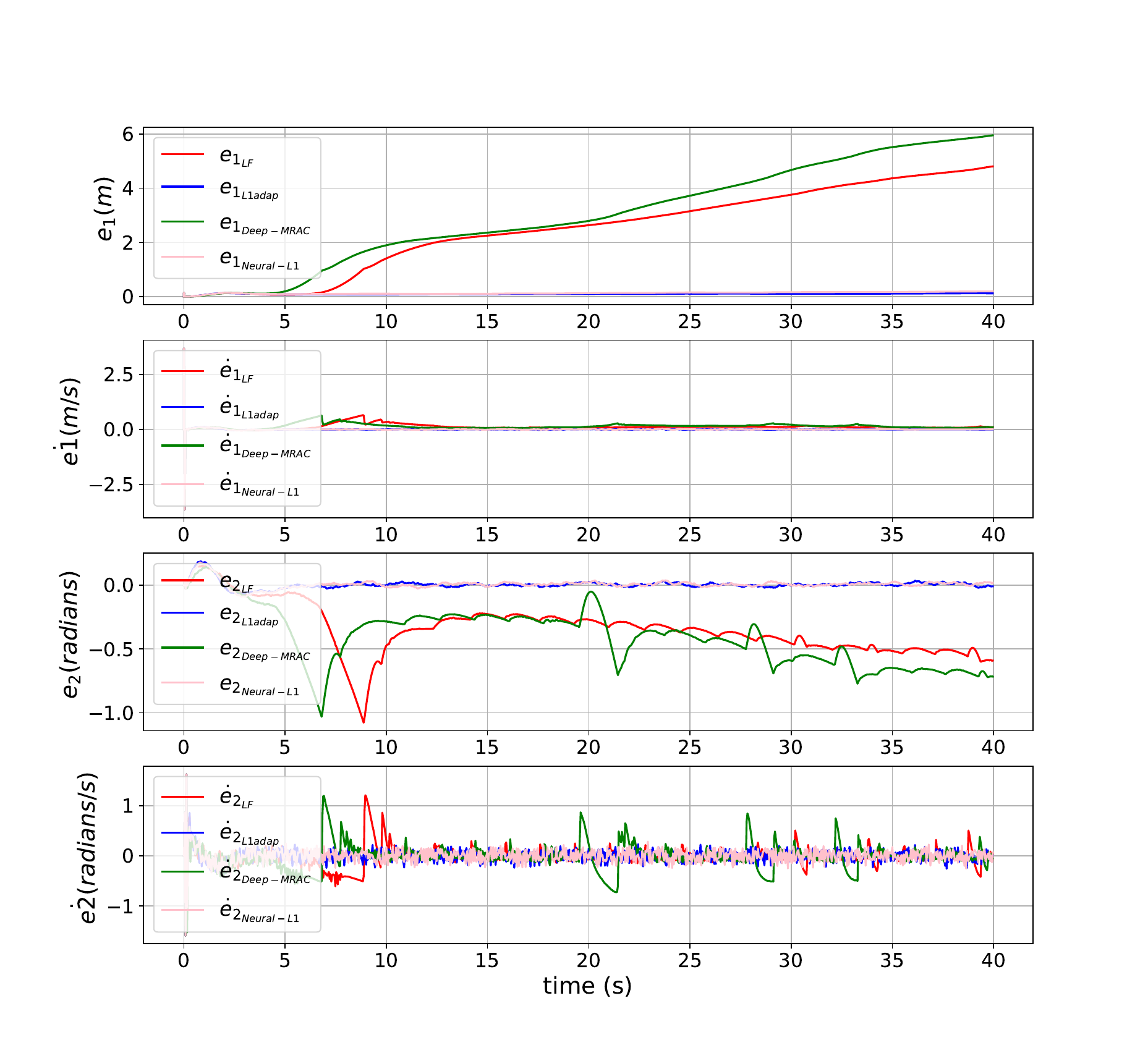}}
\caption{Profiles of the states of the system $[e_1, \dot{e}_1, e_2, \dot{e}_2]$ in PyBullet simulation. }
\vspace{-10pt}
\label{fig:state_tracking}
\end{figure}
It is clear from the \emph{green} plot in Fig. \ref{fig:uncertainty_tracking} that \deepM~controller doesn't learn the \emph{true} uncertainty represented in black, whereas controllers \ladap, \neural~are able to successfully learn the true uncertainty. This is because \ladap~and \neural controllers are robust in nature due to the selection of an appropriate strictly proper low pass filter $C(s)$, such that the $\mathcal{L}_1$-norm condition given in \eqref{eq:l1norm} is satisfied, enabling the controllers to attenuate the noise in the induced uncertainty. 



\begin{figure}
\centering
\subfloat{\includegraphics[width = 0.52\textwidth]{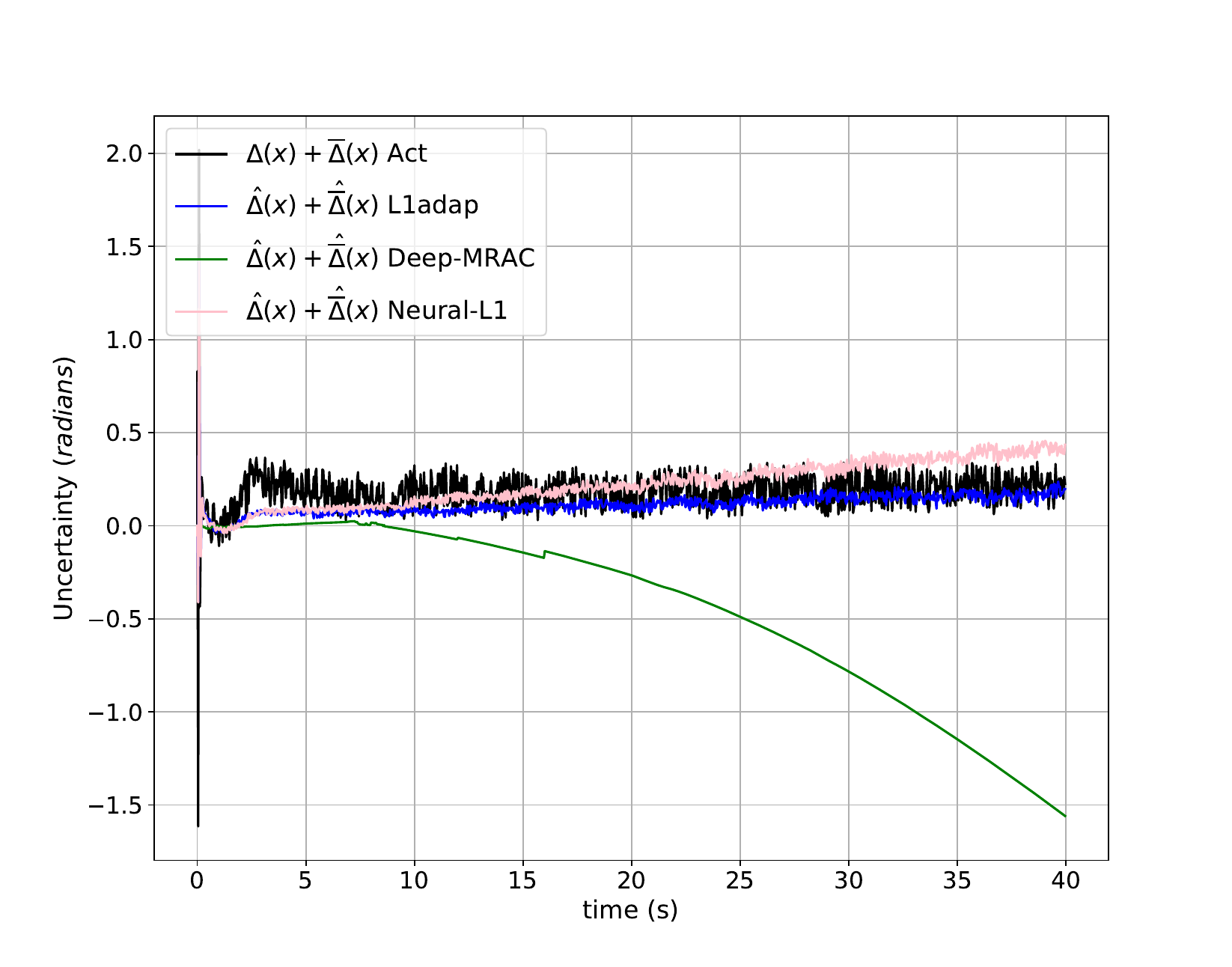}}
\caption{Comparison of uncertainty \emph{learned} by adaptive controllers for F1TENTH in PyBullet simulation. }
\vspace{-10pt}
\label{fig:uncertainty_tracking}
\end{figure}
\section{Real Experiments}\label{EXP-Results}
In this section we discuss the experimental setup and provide experimental results by evaluating our controller on the F1TENTH platform. We evaluate our proposed \neural~against the \LF, \ladap~and the \deepM~ \cite{joshi2019deep} controllers.
\subsection{Experimental Setup}
The control algorithms are deployed on an F1TENTH autonomous vehicle \cite{o2020f1tenth} running Jetson Xavier NX for onboard computation. For 6-DoF real-world state estimation, PhaseSpace X2E LED motion capture markers were attached to the autonomous vehicle which provide pose information at $50\ Hz$. ROS2 was used as the middleware stack for interfacing
between sensors, autonomous agent and control algorithms \cite{macenski2022robot}. For the F1TENTH, we test all the controllers with longitudinal velocities $V_x=0.5\ m/s$ and $V_x=1.0\ m/s$. The F1TENTH is then made to track a circular trajectory with a constant radius $2.5\ m$ and with added obstacles like a \emph{ramp} at $30^{\circ}$, and two wooden \emph{planks}. The physical obstacles induce disturbance externally, in addition to induced sensor noise and control signal disturbance. In this regard, similar to the PyBullet simulation setup described in Section~\ref{SIM-Results}, disturbance is directly injected into the control signal as constant unknown parameters $[0.5314, 0.16918, -0.6245, 0.1095]$ with added uniform white noise $[-0.1,0.1]$. In addition, we also add uniform white noise $[-0.09,0.09]$ to the PhaseSpace \emph{sensor} pose readings. Note that the constant unknown parameters and the uniform white noise added to the control signal and the uniform white noise added to PhaseSpace sensor pose readings, add up to the total \emph{true} uncertainty $\Delta(x)+\Bar{\Delta}(x)$. The complete experimental setup is shown in Fig.~\ref{fig:exp_setup_com}.
\begin{figure}
     \centering
     \begin{subfigure}[b]{0.5\textwidth}
         \centering
         \includegraphics[width=\textwidth]{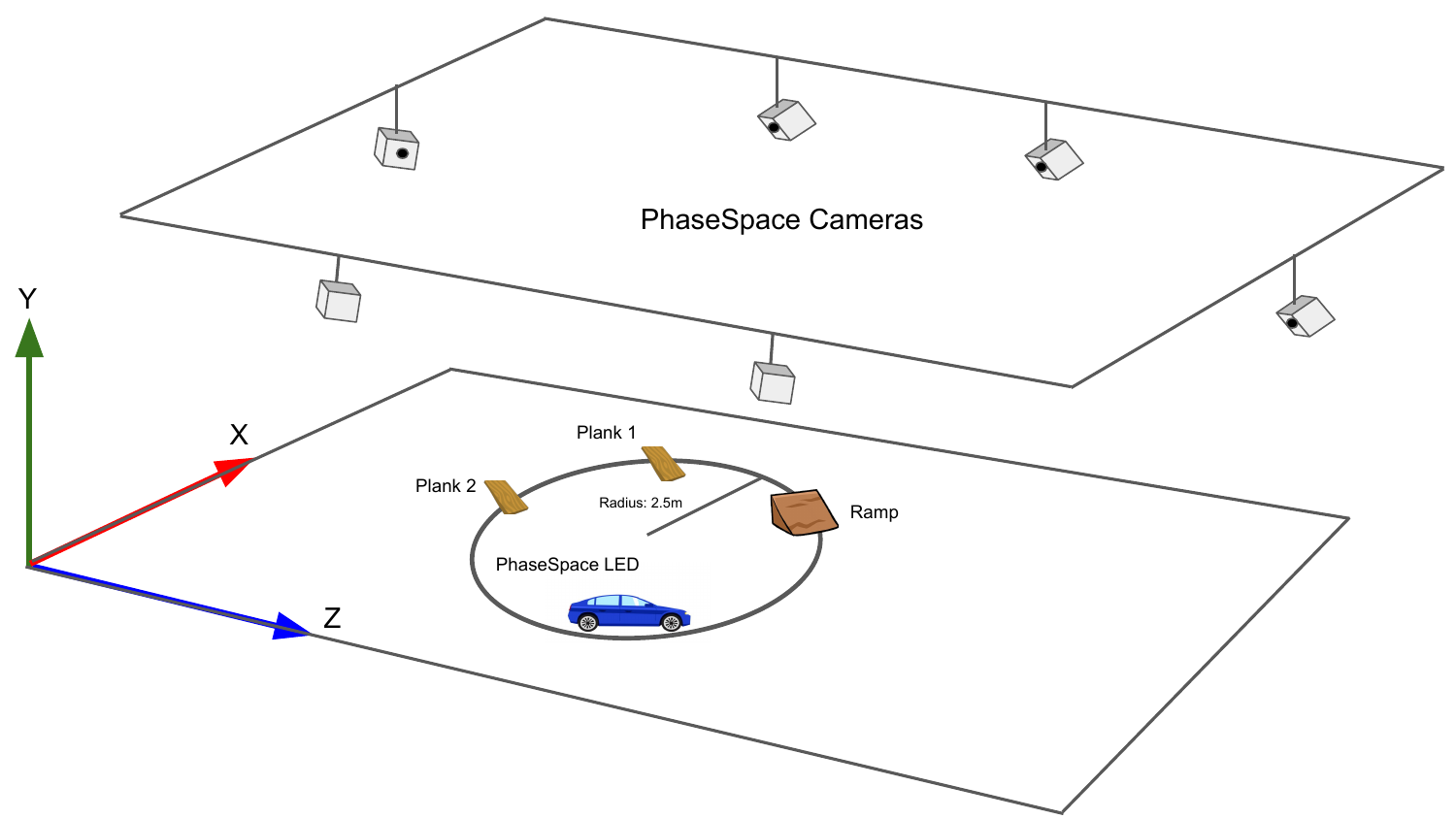}
         \caption{The MnRI Drone lab is equipped with PhaseSpace localization system that allows localization of the F1TENTH vehicle in $x,y,z$ coordinate frames.}
         \label{fig:anim_drone_lab}
     \end{subfigure}
     \hfill
     \begin{subfigure}[b]{0.5\textwidth}
         \centering
         \includegraphics[width=\textwidth]{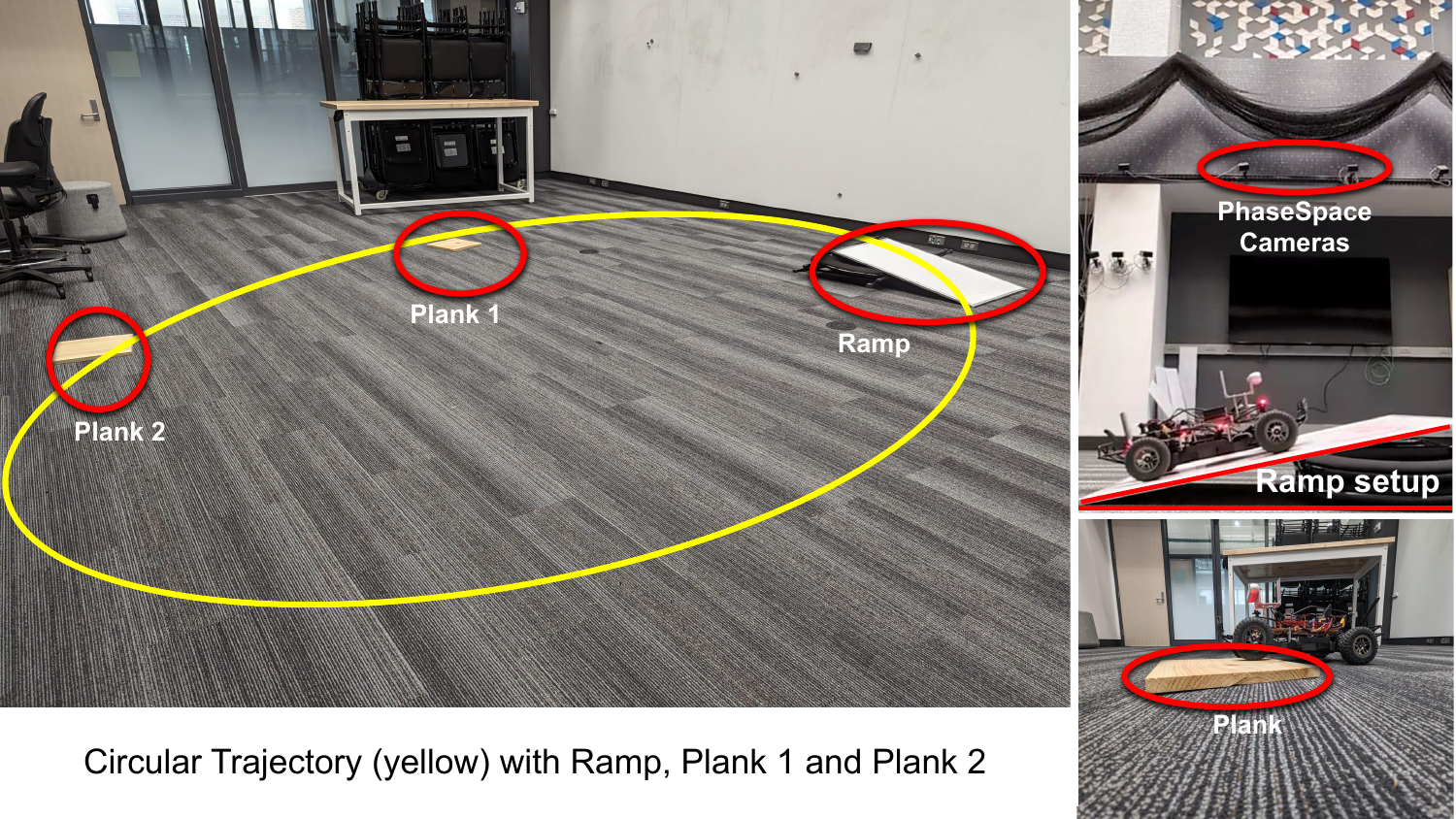}
         \caption{The circular reference trajectory (yellow) of radius $R=2.5\ m$ with physical placement of obstacles \emph{ramp} and \emph{planks}. }
         \label{fig:exp_drone_lab}
     \end{subfigure}
        \caption{Experimental setup at Minnesota Robotics Institute (MnRI) Drone lab.}
        \label{fig:exp_setup_com}
\end{figure}

\subsection{Experimental Results}

\begin{figure*}[!h]
     \centering
     \begin{subfigure}[b]{0.49\textwidth}
         \centering
         \includegraphics[width=\textwidth]{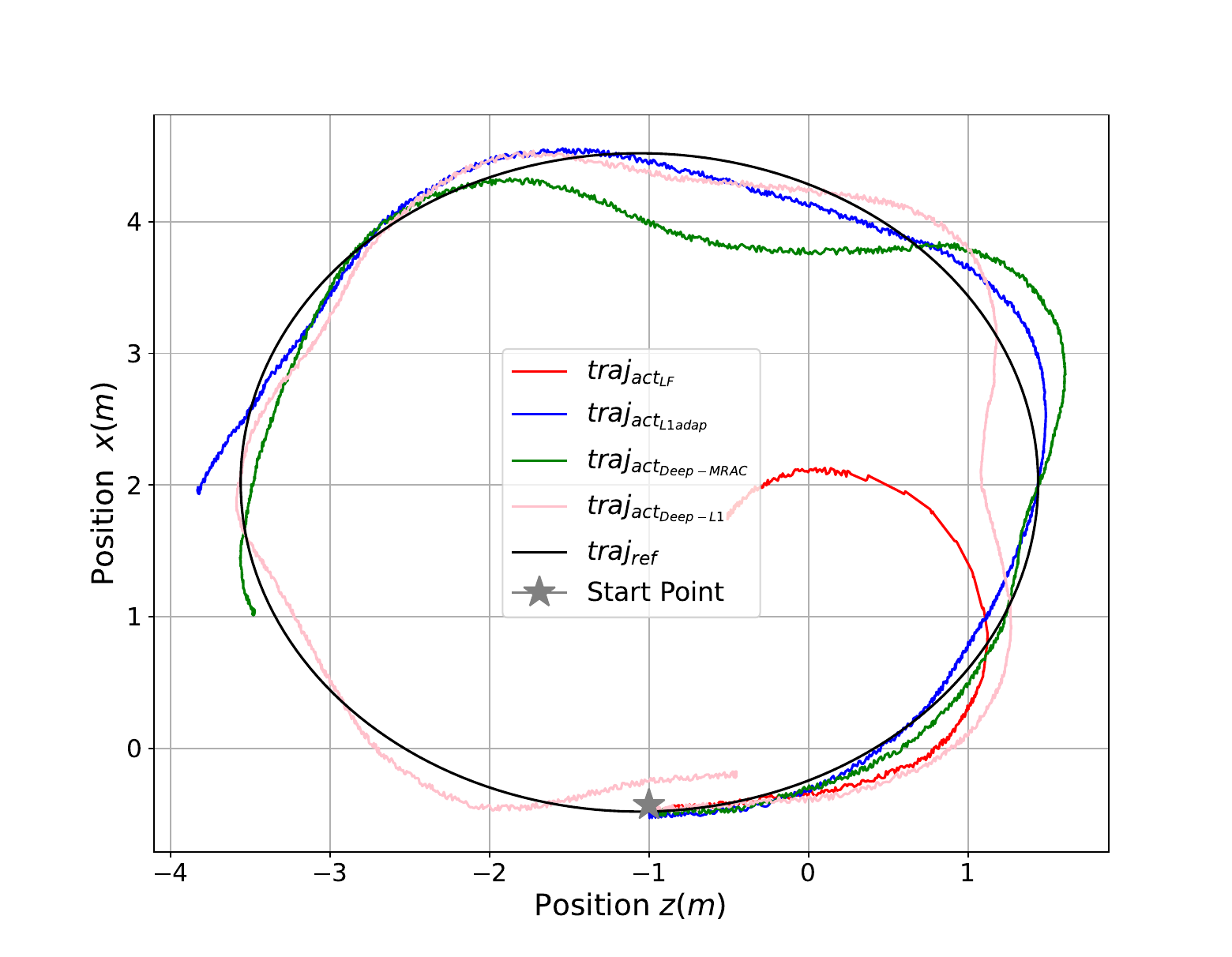}
         \caption{Trajectory tracking  performance of \neural~controller is superior to \LF, \ladap~and \deepM~controllers with $V_x=0.5m/s$.}
         \label{fig:exp_slow_traj}
     \end{subfigure}
     \hfill
     \begin{subfigure}[b]{0.49\textwidth}
         \centering
         \includegraphics[width=\textwidth]{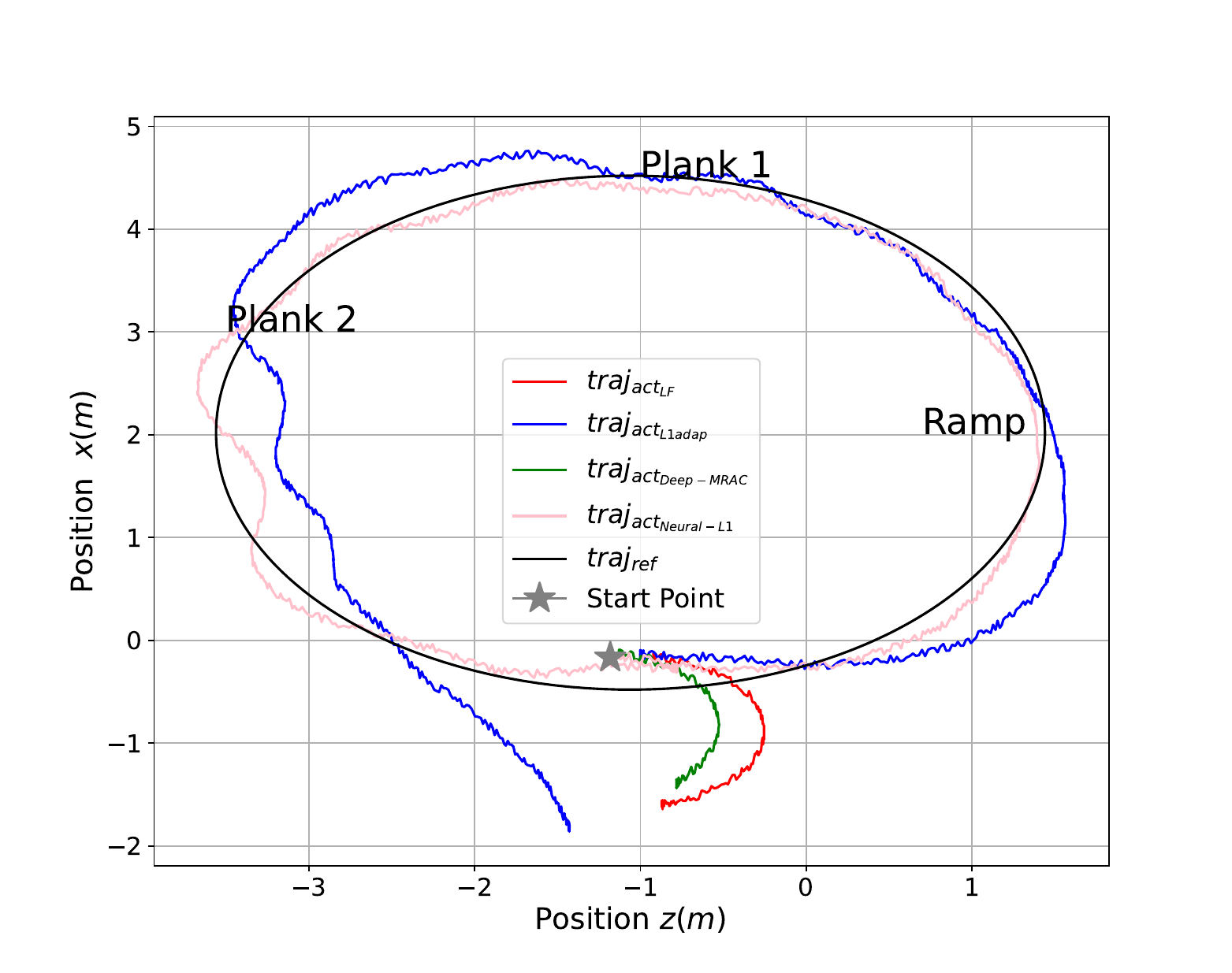}
         \caption{Trajectory tracking performance of \neural~controller is superior to \LF, \ladap~and \deepM~controllers with $V_x=1.0m/s$. }
         \label{fig:exp_fast_traj}
     \end{subfigure}
        \caption{Experiment with F1TENTH using lane-keeping dynamics on \textit{circular} trajectory for different longitudinal velocities. Unlike controller \LF(red), \ladap(blue) and \deepM(green), our proposed controller \neural(pink) is able to complete one loop of the circular trajectory with different longitudinal velocities.}
        \label{fig:real_plots}
\end{figure*}


In this subsection, we discuss in detail the results from the \emph{real-world} experiments conducted using the front-steered Ackermann vehicle F1TENTH.  Fig. \ref{fig:exp_slow_traj} shows the trajectory tracking performances of all the controllers at $V_x=0.5\ m/s$, with sensor noise and control signal disturbance, but without any physical obstacles.
 Fig.~\ref{fig:exp_fast_traj} shows the trajectory tracking performances of the controllers with \emph{ramp} and \emph{planks} added to the reference trajectory. The vehicle longitudinal velocity is increased to $V_x=1.0\ m/s$. The location of the physical obstacles, one \emph{ramp} and two \emph{planks}, are indicated in Fig.~\ref{fig:exp_fast_traj}. As shown in figures~\ref{fig:exp_slow_traj} and ~\ref{fig:exp_fast_traj}, for both experiments, \neural~tracks the trajectory exceptionally well and completes one loop of the circular trajectory, whereas \LF, \ladap~and \deepM~controllers are not able to even complete one loop of the trajectory. 
Next, in figures \ref{fig:exp_fast_state_tracking} and \ref{fig:exp_fast_uncertainty_tracking}, we study the profiles of the \emph{states}, $[e_1, \dot{e}_1, e_2, \dot{e}_2]$, of the lateral error dynamics from \cite{rajamani2011vehicle} and the \emph{learned} uncertainty of the adaptive controllers corresponding to the circular trajectory shown in Fig.~\ref{fig:exp_fast_traj}.
In Fig.~\ref{fig:exp_fast_state_tracking}, the states \emph{lateral error} $e_1$ and \emph{yaw angle error} $e_2$ deviate from the reference for \LF~controller and the \deepM~controller, implying unstable behavior of the controllers. Only for controllers \neural~and \ladap, all four states remain in proximity to zero. Further, when the tracking error profiles of \neural~and \ladap~ are compared in Fig~\ref{fig:exp_fast_state_tracking}, the errors $e_1$ and $e_2$ converge better for \neural. This observation is not surprising because \neural~controller is inherently based on the theory of \ladap. However, ultimately \ladap~controller becomes unstable leading the F1TENTH vehicle to deviate from the \textit{ref} trajectory. The deviation of the F1TENTH vehicle, due to the unstable behavior of \ladap~controller, from the \textit{ref} trajectory is clearly indicated by the increase in the magnitude of the state profiles $e_1$ and $e_2$ plotted in blue in Fig.~\ref{fig:exp_fast_state_tracking} after $17.5\ s$.

\begin{figure}
\centering
\subfloat{\includegraphics[width = 0.52\textwidth]{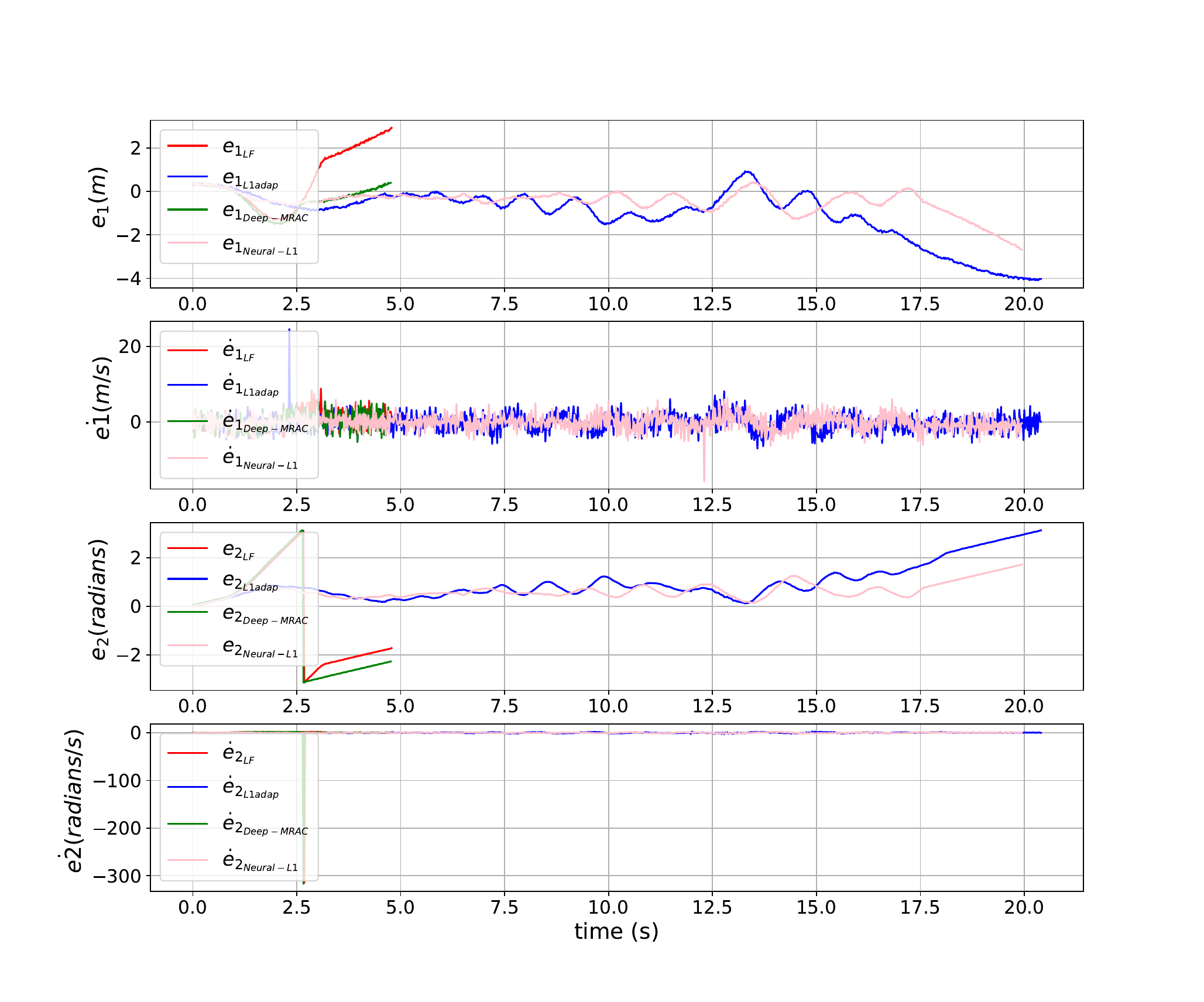}}
\caption{Comparison of of the states of the system $[e_1, \dot{e}_1, e_2, \dot{e}_2]$ for F1TENTH tracking a circular trajectory of radius $2.5m$ at a longitudinal speed of $1.0m/s$.}
\vspace{-10pt}
\label{fig:exp_fast_state_tracking}
\end{figure}
These results can be better understood by studying the \emph{learned} uncertainty profiles in Fig.~\ref{fig:exp_fast_uncertainty_tracking}. From Fig.~\ref{fig:exp_fast_uncertainty_tracking}, we notice that the uncertainty learned by \neural~converges close to the true uncertainty $\Delta(x)+\Bar{\Delta}(x)ACT$, and the uncertainty learned by \ladap~deviates at the end of the experiment. This is due to the external obstacles like \emph{ramp} and \emph{planks} added to the trajectory. These results imply that even though both \ladap~and \neural~controllers are \emph{robust} in nature, the neural network is able to learn the uncertainties induced into the system by the presence of physical obstacles such as the \emph{ramp} and the two \emph{planks}, giving the \neural~controller a better capability than the \ladap~to learn the overall uncertainty. The oscillations induced in the trajectory, \emph{blue} plot in Fig.~\ref{fig:exp_fast_traj}, of the \ladap~controller, right after the F1TENTH travels past \emph{plank} 2, corroborate the observation in Fig.~\ref{fig:exp_fast_uncertainty_tracking} that the learned uncertainty does not match the true uncertainty for \ladap. These oscillations lead the F1TENTH vehicle to not even complete one loop of the circular \textit{ref} trajectory.  
\begin{figure}
\centering
\subfloat{\includegraphics[width = 0.52\textwidth]{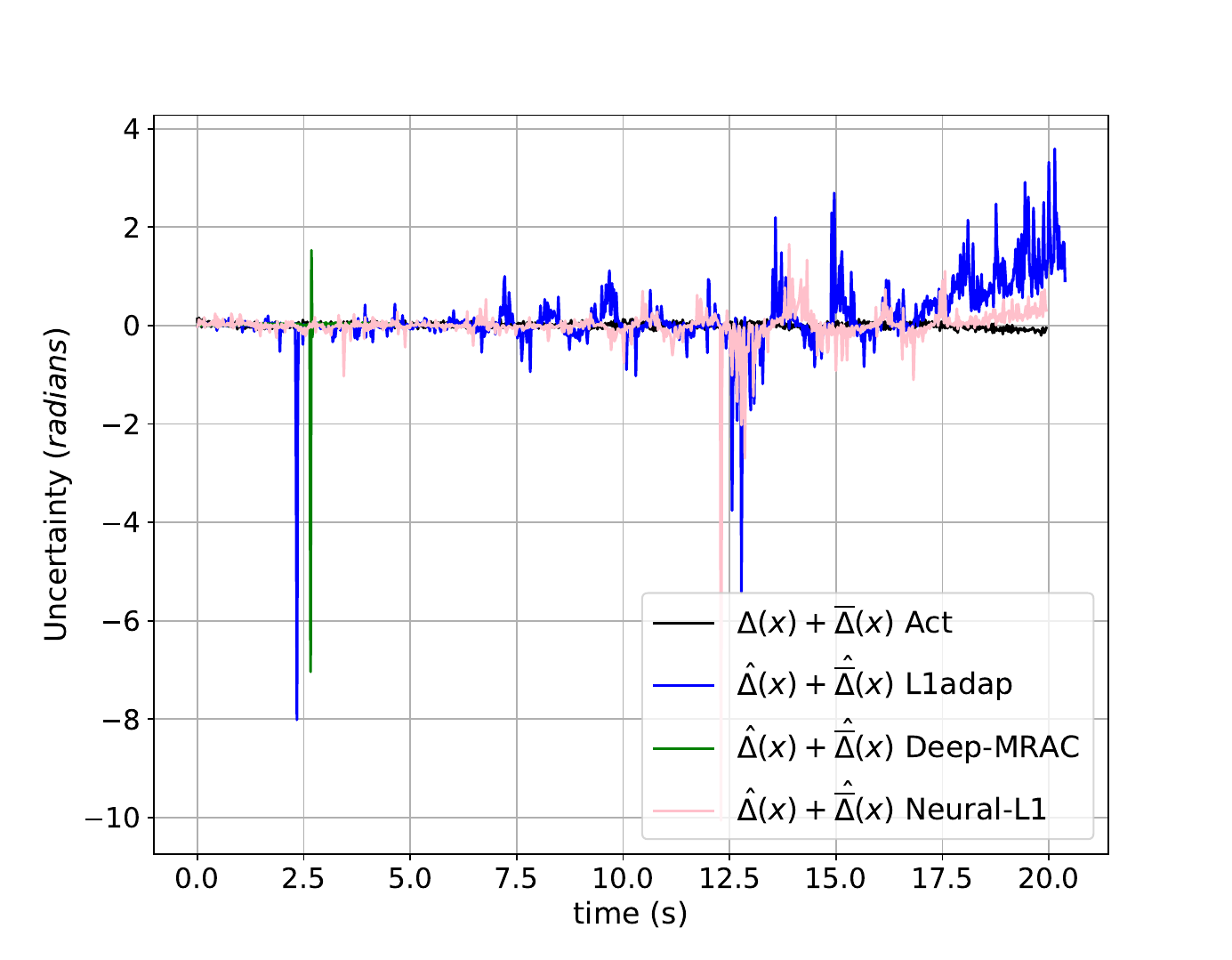}}
\caption{Uncertainty \emph{learned} by adaptive controllers for F1TENTH tracking a circular trajectory of radius $2.5m$ at a longitudinal speed of $1.0m/s$. }
\vspace{-10pt}
\label{fig:exp_fast_uncertainty_tracking}
\end{figure}

The increase in oscillations in the trajectory is further corroborated in $e_1$ and $e_2$ plots in Fig.~\ref{fig:exp_fast_state_tracking}, after about $12.5\ s$. Clearly, \neural~controller handles these oscillations much better in comparison to \ladap~controller, by keeping the \emph{errors} close to zero. These results conclude that the \neural~controller is stable as well as robust to disturbance and noises induced externally to the system, even when the disturbances are in the form of physical objects.
For more extensive experimental results, we refer the reader to the video submission, where the superior tracking performance of the \neural~ controller is further exhibited by testing the \neural~controllers over longer duration runs, with F1TENTH running over multiple loops of the circular \textit{ref} trajectory. Our project page, including supplementary material and videos, can be found at \url{ https://mukhe027.github.io/Neural-Adaptive-Control/}


\section{Conclusions}\label{end}

In this paper, we addressed the problem of stable and robust control of front-steered Ackermann vehicles with lateral error dynamics for the application of \emph{lane keeping}. We demonstrated that, in the presence of uncertainties induced into the control system, most existing controllers such as linear \emph{state-feedback} controller (i.e. PID), $\mathcal{L}_1$ Adaptive controller \cite{hovakimyan20111} and the \deepM~\cite{joshi2019deep} fail to track a reference trajectory. This is particularly due to the controllers' inabilities to attenuate noise.  Out of the three existing controllers, linear \emph{state-feedback} controller, $\mathcal{L}_1$ Adaptive controller \cite{hovakimyan20111} and the \deepM~\cite{joshi2019deep}, $\mathcal{L}_1$ Adaptive controller \cite{hovakimyan20111} exhibits the best tracking performance because unlike other controllers, $\mathcal{L}_1$ Adaptive controller\cite{hovakimyan20111} is designed traditionally to be stable, adaptive as well as \emph{robust}. However, under the circumstance of uncertainties that are induced as sensor noise, control signal disturbance and physical obstacles on the reference trajectory, even $\mathcal{L}_1$ Adaptive controller \cite{hovakimyan20111} fails. Therefore, taking inspiration from the stability, adaptiveness and the robustness properties of the $\mathcal{L}_1$ Adaptive controller \cite{hovakimyan20111} and combining it with the novel neural network approach of learning the adaptive laws in \deepM~\cite{joshi2019deep}, we introduced the Neural $\mathcal{L}_1$ Adaptive controller. 

We demonstrated the effectiveness of the Neural $\mathcal{L}_1$ Adaptive controller by taking a multi-tiered approach: i) We extended the theory of guaranteed stability, adaptiveness and robustness of $\mathcal{L}_1$ Adaptive controller \cite{hovakimyan20111}, while combining it with the neural network approach of \deepM~\cite{joshi2019deep}, to the Neural $\mathcal{L}_1$ Adaptive controller; ii) We conducted extensive physics-based simulation on PyBullet using a model of F1TENTH vehicle and compared the trajectory tracking performance of our proposed controller with the performance of other existing controllers; iii) We conducted \emph{real} experiments using a F1TENTH platform and evaluated the performance of our controller. 
These results allowed us to conclude that the Neural $\mathcal{L}_1$ adaptive controller behaves in a stable and robust manner in the presence of uncertainties induced into the dynamics and has \emph{superior} tracking performance than a traditional linear state-feedback \cite{rajamani2011vehicle}, the \deepM~controller introduced in \cite{joshi2019deep}, and the conventional $\mathcal{L}_1$ adaptive controller \cite{hovakimyan20111}.
Future work could extend to use perception to learn features from the environment which in turn could be used to learn the true uncertainty more accurately and expand
the stability results when the robots have noise from multiple sensors.

\bibliographystyle{IEEEtran}
\bibliography{ref.bib}

\begin{IEEEbiography}[{\includegraphics[width=1in,height=1.25in,clip,keepaspectratio]{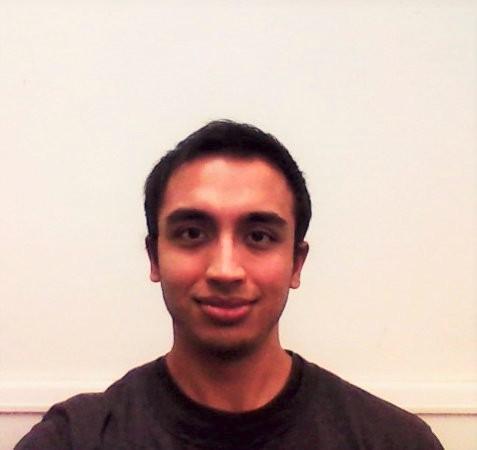}}]{Pratik Mukherjee} received his
B.S. and M.S. degree in mechanical engineering from University of Minnesota Twin-Cities, Minneapolis, USA,
in 2014 and 2016, respectively. He obtained his
Ph.D. degree in electrical engineering from Virginia
Tech, Blacksburg, VA, USA, as part of the Coordination at Scale Laboratory. He is currently working as Postdoctoral Associate at the University of Minnesota Twin-Cities, Minneapolis, USA in the computer science department in the Robotic Sensor Networks Lab.
His current research interests include distributed, stable and robust topology control applied to single or multi-robot systems. He is an incoming Assistant Professor at Florida Atlantic University, Boca Raton, FLorida in the Ocean and Mechanical Engineering Department. 
\end{IEEEbiography}

\begin{IEEEbiography}[{\includegraphics[width=1in,height=1.25in,clip,keepaspectratio]{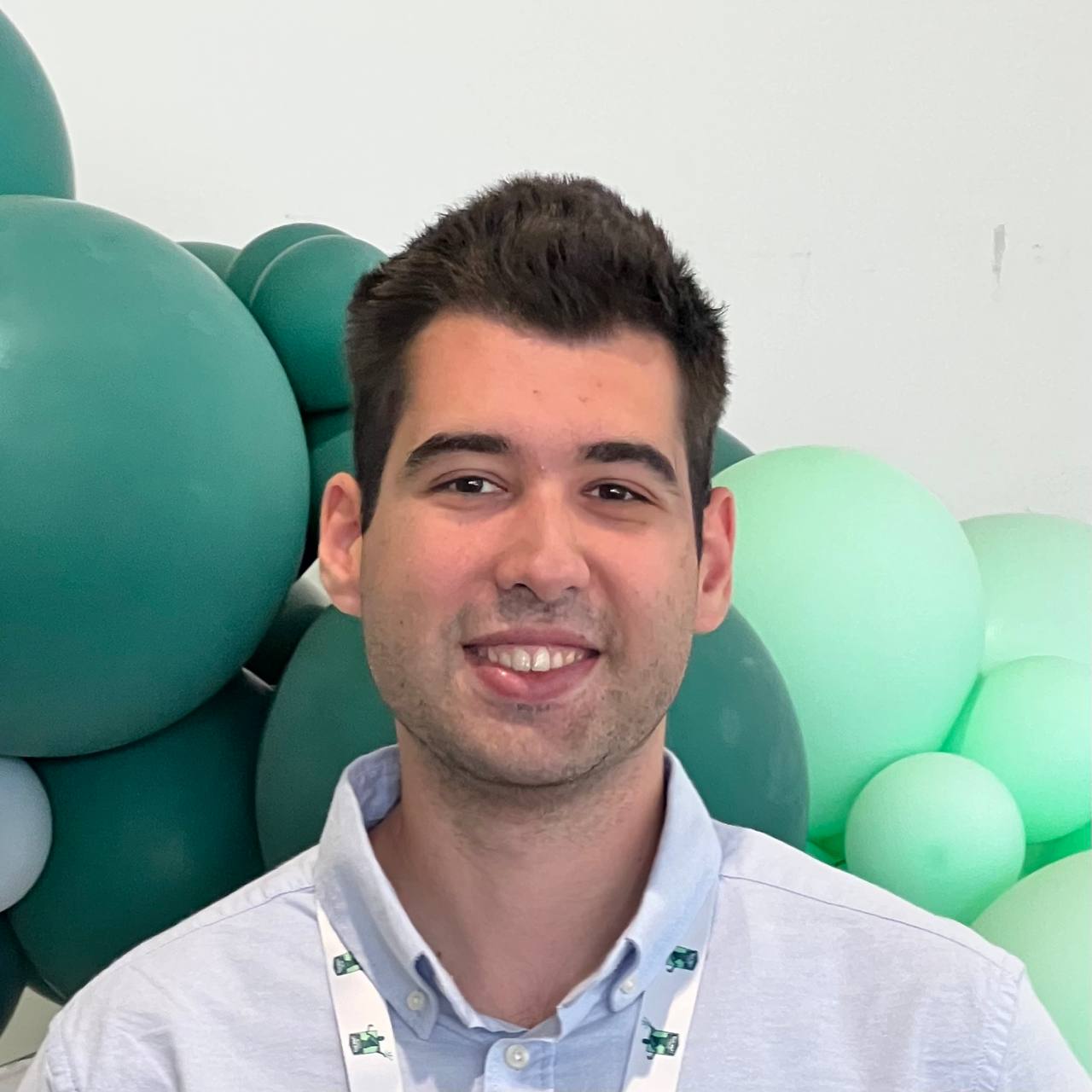}}]{Burak M. Gonultas} received his B.S. degrees in electronics and communication engineering and computer engineering from Istanbul Technical University, Istanbul, Turkey, in 2018 and 2019, respectively. As a Ph.D. candidate he is currently working as Research Assistant at the University of Minnesota Twin-Cities, Minneapolis, USA in the computer science department in the Robotic Sensor Networks Lab. His current research interests include planning, dynamics, control and multi-agent reinforcement learning for autonomous mobile systems.
\end{IEEEbiography}

\begin{IEEEbiography}[{\includegraphics[width=1in,height=1.25in,clip,keepaspectratio]{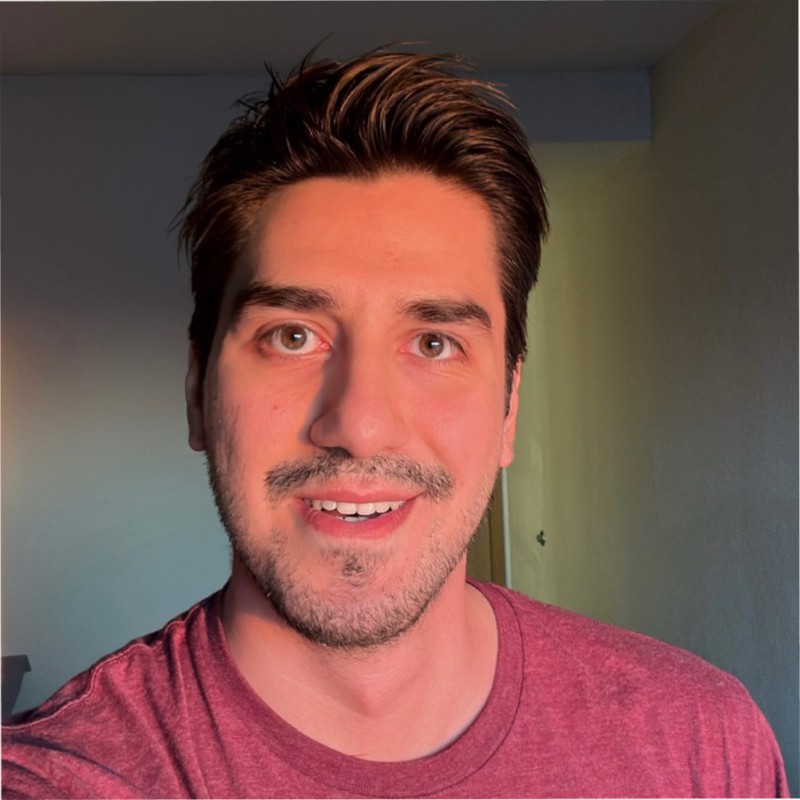}}]{O. Goktug Poyrazoglu}
received his
B.S. degree in mechanical engineering from Bilkent University, Ankara, Turkey,
in 2020. He is currently working as Ph.D. Student at the University of Minnesota Twin-Cities, Minneapolis, USA in the computer science department in the Robotic Sensor Networks Lab.
His current research interests include trajectory optimization, quantifying uncertainty, and modeling environmental interactions for mobile robotic applications.
\end{IEEEbiography}

\begin{IEEEbiography}[{\includegraphics[width=1in,height=1.25in,clip,keepaspectratio]{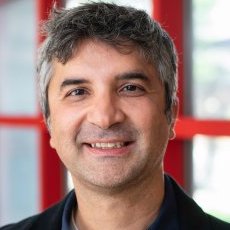}}]{Volkan Isler} received the B.S. degree in computer engineering from Bogazici University, Istanbul, Turkey,
in 1999, and the M.S.E. and Ph.D. degrees in computer and information science from University of
Pennsylvania, Philadelphia, PA, USA, in 2000 and
2004, respectively.
He is currently a Professor in the Department of Computer Science \& Engineering with University of
Minnesota, Minneapolis, MN, USA where he directs the Robotic Sensor Networks Lab. From 2020 to 2023, he was the head of Samsung's AI Center in NY. His research interests include geometric algorithms for perception and planning, and their applications to agricultural and consumer robots. From 2009 to 2015, he was the Chair of the Robotics and
Automation Society Technical Committee on networked robots. He received the
CAREER Award from the National Science Foundation, and McKnight Land-grant professorship from the University of Minnesota.
\end{IEEEbiography}

\end{document}